\documentclass{article} 
\usepackage{iclr2027_conference,times}

\usepackage{url}

\usepackage{hyperref}
\hypersetup{
	colorlinks=true,
	linkcolor=black,
	citecolor=black,
	urlcolor=black
}

\usepackage{amsmath}
\usepackage{amssymb}
\usepackage{amsfonts}
\usepackage{mathtools}
\usepackage{amsthm}

\theoremstyle{plain}
\newtheorem{theorem}{Theorem}[section]
\newtheorem{proposition}[theorem]{Proposition}

\usepackage{enumitem}

\usepackage{wrapfig}
\usepackage{graphicx}
\usepackage{subcaption}

\title{A Dual Representation of Influence Functions for Linearizable Models}

\author{Zhenhuan Sun \\
	University of Toronto \\
	\texttt{zhsun@ece.utoronto.ca}
	\And
	Shahrokh Valaee \\
	University of Toronto \\
	\texttt{valaee@ece.utoronto.ca}
}

\iclrfinalcopy 
\begin{document}

	\maketitle
	\lhead{Preprint}
	
	\begin{abstract}
		In this paper, we present a dual representation of influence functions, whose computational complexity scales with dataset size rather than model size. Both analytically and experimentally, we show that this representation can be an efficient alternative to the original influence functions for estimating changes in parameters, model outputs and loss due to data point removal, when model size is large relative to dataset size, or when evaluating the original influence functions in parameter space is infeasible. The dual representation, however, is limited to linearizable models, which are models whose behavior can be approximated by their linearizations throughout training, and requires materializing a matrix, whose size grows with the product of model output dimension and dataset size.
	\end{abstract}
	
	\section{Introduction}
	
	Interpretability and privacy are critical considerations in many machine learning applications. However, many modern machine learning models, particularly deep neural networks, lack transparency and may inadvertently memorize undesirable information or expose sensitive information to users. This creates challenges for understanding model behavior, and raises concerns about the reliability of model predictions and the risk of exposing private information. As a result, there is growing interest in developing methods that improve interpretability and provide mechanisms that remove the influence of sensitive or undesirable information from trained models.
	
	Influence functions, originally developed in robust statistics \citep{hampel1974influence} and later introduced to machine learning by \citet{koh2017understanding}, aim to address these issues by providing a way of estimating how model parameters and parameter-dependent functions change when training data points are perturbed. This ability has made influence functions useful for understanding model behavior, where they can be used to identify the training data points whose presence is estimated to have the largest effect on a particular prediction or loss, as well as for machine unlearning \citep{nguyen2025survey}, which is an emerging field that aims to efficiently remove the influence of specific data points from a trained model so that the resulting model behaves as if it had been retrained from scratch without those data points. Despite being successfully used to unlearn and understand the behavior of small models \citep{koh2017understanding, guo2020certified}, influence functions cannot be readily scaled to large models due to the need to compute an inverse Hessian vector product (IHVP), whose computational complexity grows with the number of model parameters and the cost of evaluating model. 
	
	Due to this increasing computational complexity, exact IHVP evaluation can become computationlly infeasible for large models, or, even when it remains feasible, its cost can eventually exceed the cost of retraining the model, eliminating the computational advantage of using influence functions. The inability to evaluate the IHVP both quickly and exactly for large models has led to sustanined interest in scaling influence functions by approximating the IHVP, either directly or through an approximation of the Hessian matrix \citep{koh2017understanding, schioppa2022scaling, grosse2023studying}. However, these approaches can reduce the reliability of influence functions because they introduce additional approximation error into a quantity that is already an approximation. The resulting error therefore consists of both the intrinsic approximation error of the influence function and the additional error introduced by approximating the IHVP.
	
	Given the current state of influence function research, the use of influence functions thus faces a dilemma. On the one hand, obtaining reliable influence function estimates requires evaluating the IHVP exactly, but exact evaluation is only practical for relatively small models. On the other hand, scaling influence functions to larger models generally requires approximating the IHVP, which can introduce additional approximation error and reduce the reliability of the resulting influence function estimates. Motivated by this dilemma, we investigate whether influence functions can be re-represented in a way that makes them inherently more scalable, in other words, whether a new representation can be fround that allows influence functions to be evaluated both quickly and exactly for large models without approximating the IHVP. Inspired by the dual formulation of linear regression, in this work, we introduce a dual representation of the influence functions for linearizable models, which shifts the computational complexity of IHVP from scaling with model size to scaling with dataset size. We show that the dual representation can be a more efficient alternative to the original representation when model size is large relative to dataset size, and that it provides a means of estimating changes in model outputs and loss due to data point removal for models with infinitely many parameters.
	
	\section{Preliminaries} \label{sec:Preliminaries}
	
	This section introduces our notations and provides conceptual background needed for the development of the dual representation of influence functions. The derivations of the results in this section are provided in Appendix \ref{A:Derivations}.
	
	Let \( \mathcal{D} \subseteq \mathcal{X} \times \mathcal{Y} \) be a training dataset of input-output pairs \( (\mathbf{x} , \mathbf{y}) \), where \( \mathcal{X} \subseteq \mathbb{R}^{d_\text{in}} \) and \( \mathcal{Y} \subseteq \mathbb{R}^{d_\text{out}} \) denote input and output spaces. Let \( f: \mathbb{R}^{d_\text{in}} \times \mathbb{R}^{d_{\boldsymbol{\theta}}} \to \mathbb{R}^{d_\text{out}} \) be a parameterized machine learning model, and consider the regularized empirical risk over \( \mathcal{D} \)
	\begin{equation*}
		\mathcal{L}_\mathcal{D}(\boldsymbol{\theta}) \coloneqq \frac{1}{|\mathcal{D}|} \sum_{(\mathbf{x}, \mathbf{y}) \in \mathcal{D}} \left( \ell(f(\mathbf{x}, \boldsymbol{\theta}), \mathbf{y}) + \frac{\lambda}{2} \| \boldsymbol{\theta} \|_2^2 \right),
	\end{equation*}
	where \( \lambda > 0 \) denotes a regularization constant and \( \ell: \mathbb{R}^{d_\text{out}} \times \mathbb{R}^{d_\text{out}} \to \mathbb{R}_{\geq 0} \) denotes a loss function. When \( \ell \) is convex in \( \boldsymbol{\theta} \), \( \mathcal{L}_\mathcal{D}(\boldsymbol{\theta}) \) is strongly convex and has a unique minimizer \( \boldsymbol{\theta}^\star \coloneqq \arg \, \min_{\boldsymbol{\theta} \in \mathbb{R}^{d_{\boldsymbol{\theta}}}} \mathcal{L}_{\mathcal{D}}(\boldsymbol{\theta}) \) for any dataset \( \mathcal{D} \). Let \( \mathbf{X} \in \mathbb{R}^{|\mathcal{D}| \times d_\text{in}} \) denotes the feature matrix of \( \mathcal{D} \), the gradient and Hessian of such empirical risk with respect to vectorized model outputs \( f(\mathbf{X}, \boldsymbol{\theta}) \coloneq \operatorname{vec}(\{ f(\mathbf{x}, \boldsymbol{\theta}) \}_{\mathbf{x} \in \mathcal{D}}) \in \mathbb{R}^{d_\text{out}|\mathcal{D}|} \) are denoted as \( \nabla_{f(\mathbf{X}, \boldsymbol{\theta})} \mathcal{L}_{\mathcal{D}} \) and \( \nabla_{f(\mathbf{X}, \boldsymbol{\theta})}^2 \mathcal{L}_{\mathcal{D}} \) respectively.
	
	
	\subsection{Influence Functions} \label{subsec:Influence Functions}
	
	Let \( \mathcal{D}_f\) denote a forget dataset, which contains data points that need to be deleted from \( \mathcal{D} \) and whose influence must be removed from \( \boldsymbol{\theta}^\star \). The remaining data points form a retain dataset, which is denoted as \( \mathcal{D}_r\). Assume, without loss of generality, that data points in \( \mathcal{D}_f \) are the first \( |\mathcal{D}_f| \) data points in \( \mathcal{D} \). Consider an upweighted version of \( \mathcal{L}_\mathcal{D}(\boldsymbol{\theta}) \), defined as \( \mathcal{L}_{\mathcal{D}}(\boldsymbol{\theta}, -\frac{1}{|\mathcal{D}|}) \coloneqq \mathcal{L}_{\mathcal{D}}(\boldsymbol{\theta}) - \frac{|\mathcal{D}_f|}{|\mathcal{D}|}\mathcal{L}_{\mathcal{D}_f}(\boldsymbol{\theta}) \), where the contributions of the data points in \( \mathcal{D}_f \) are upweighted by \( -\frac{1}{|\mathcal{D}|} \).
	
	The idea of influence functions is that, instead of solving \( \boldsymbol{\theta}_r^\star \coloneqq \arg \, \min_{\boldsymbol{\theta} \in \mathbb{R}^{d_{\boldsymbol{\theta}}}} \mathcal{L}_{\mathcal{D}_r}(\boldsymbol{\theta}) \) from scratch, one can use the first-order Taylor approximation to approximate \( \boldsymbol{\theta}_r^\star  \) by
	\begin{equation}
		\boldsymbol{\theta}_r^\star \approx \boldsymbol{\theta}^\star + \frac{|\mathcal{D}_f|}{|\mathcal{D}|} \mathbf{H}_{\boldsymbol{\theta}^\star}^{-1} \nabla_{\boldsymbol{\theta}}\mathcal{L}_{\mathcal{D}_f}(\boldsymbol{\theta}^\star),
		\label{eq:influence on theta}
	\end{equation}
	where \( \mathbf{H}_{\boldsymbol{\theta}^\star} \coloneqq \nabla_{\boldsymbol{\theta}}^2\mathcal{L}_\mathcal{D}(\boldsymbol{\theta}^\star, -\frac{1}{|\mathcal{D}|}) \in \mathbb{R}^{d_{\boldsymbol{\theta}} \times d_{\boldsymbol{\theta}}}\) is the Hessian of the upweighted empirical risk, and \( \mathbf{H}_{\boldsymbol{\theta}^\star}^{-1} \) may be further approximated by \( \nabla_{\boldsymbol{\theta}}^2\mathcal{L}_\mathcal{D}(\boldsymbol{\theta}^\star)^{-1} \) when \( \frac{|\mathcal{D}_f|}{|\mathcal{D}|} \) is small. In addition, for any differentiable function \( g: \mathbb{R}^{d_{\boldsymbol{\theta}}} \to \mathbb{R}^k \), \( g(\boldsymbol{\theta}_r^\star)  \) can be approximated by evaluating the first-order Taylor expansion of \( g(\boldsymbol{\theta}) \) around \( \boldsymbol{\theta}^\star \) at \( \boldsymbol{\theta}_r^\star \). For example, let \( \mathbf{z}_t \coloneq (\mathbf{x}_t, \mathbf{y}_t) \) denote a test data point, the change in model output and loss at \( \mathbf{z}_t \) caused by removing \( \mathcal{D}_f \) can be approximated by
	\begin{align}
		&f(\mathbf{x}_t, \boldsymbol{\theta}_r^\star) - f(\mathbf{x}_t, \boldsymbol{\theta}^\star) \approx \nabla_{\boldsymbol{\theta}} f(\mathbf{x}_t, \boldsymbol{\theta}^\star) (\boldsymbol{\theta}_r^\star - \boldsymbol{\theta}^\star) \label{eq:change in output theta}\\
		&\mathcal{L}_{\{\mathbf{z}_t\}}(\boldsymbol{\theta}_r^\star) - \mathcal{L}_{\{\mathbf{z}_t\}}(\boldsymbol{\theta}^\star) \approx \nabla_{\boldsymbol{\theta}} \mathcal{L}_{\{\mathbf{z}_t\}}(\boldsymbol{\theta}^\star)^{\top} (\boldsymbol{\theta}_r^\star - \boldsymbol{\theta}^\star), \label{eq:change in loss theta}
	\end{align}
	where \( \boldsymbol{\theta}_r^\star - \boldsymbol{\theta}^\star \) is given by \eqref{eq:influence on theta}.
	
	The computational bottleneck of \eqref{eq:influence on theta} is the IHVP: \( \mathbf{H}_{\boldsymbol{\theta}^\star}^{-1} \nabla_{\boldsymbol{\theta}}\mathcal{L}_{\mathcal{D}_f}(\boldsymbol{\theta}^\star) \). Since \( \mathbf{H}_{\boldsymbol{\theta}^\star} \) scales with the number of model parameters, materializing it is prohibitively expensive for large models. This rules out analytical inversion or the use of standard linear system solvers that require access to the full matrix. Consequently, exact computation of the IHVP typically relies on the conjugate gradient (CG) method, which solves the linear system 
	\begin{equation}
		\mathbf{H}_{\boldsymbol{\theta}^\star} \mathbf{s} = \frac{|\mathcal{D}_f|}{|\mathcal{D}|}\nabla_{\boldsymbol{\theta}}\mathcal{L}_{\mathcal{D}_f}(\boldsymbol{\theta}^\star)
		\label{eq:linear system in theta}
	\end{equation}
	iteratively using only a linear operator \( \mathbf{v} \mapsto \mathbf{H}_{\boldsymbol{\theta}^\star} \mathbf{v} \) that evaluates the Hessian vector products (HVPs), without explicitly forming the Hessian matrix. 
	
	
	\section{Methodology} \label{sec:Methodology}
	
	This section derives the dual representation of influence functions for linearized models, and discusses the advantages of the dual representation. The proofs and derivations of the results in this section are provided in Appendices \ref{A:Proofs} and \ref{A:Derivations}.
	
	\subsection{A Dual Formulation} \label{subsec: A Dual Formulation}
	
	Consider the linearization of \( f(\cdot, \boldsymbol{\theta}) \) around \( \boldsymbol{\theta}' \), defined as \( f^\text{lin}(\cdot, \boldsymbol{\theta}) \coloneq  f(\cdot, \boldsymbol{\theta}') + \nabla_{\boldsymbol{\theta}} f(\cdot, \boldsymbol{\theta}') (\boldsymbol{\theta} - \boldsymbol{\theta}') \), and suppose it is trained by minimizing the regularized empirical risk
	\begin{equation*}
		\hat{\mathcal{L}}_\mathcal{D}(\boldsymbol{\theta}) \coloneqq \frac{1}{|\mathcal{D}|} \sum_{(\mathbf{x}, \mathbf{y}) \in \mathcal{D}} \left( \ell(f^\text{lin}(\mathbf{x}, \boldsymbol{\theta}), \mathbf{y}) + \frac{\lambda}{2} \| \boldsymbol{\theta} - \boldsymbol{\theta}' \|_2^2 \right).
	\end{equation*}
	Define \( \hat{\boldsymbol{\theta}}^\star \coloneq \arg \, \min_{\boldsymbol{\theta} \in \mathbb{R}^{d_{\boldsymbol{\theta}}}} \hat{\mathcal{L}}_{\mathcal{D}}(\boldsymbol{\theta}) \) and consider the reparameterization
	\begin{equation*}
		\boldsymbol{\theta} \coloneq \phi(\Delta \boldsymbol{\alpha}) \coloneq \hat{\boldsymbol{\theta}}^\star + \nabla_{\boldsymbol{\theta}} f\left(\mathbf{X}, \boldsymbol{\theta}'\right)^\top \Delta \boldsymbol{\alpha},
	\end{equation*}
	where \( \Delta \boldsymbol{\alpha} \in \mathbb{R}^{d_\text{out}|\mathcal{D}|} \) and \( \nabla_{\boldsymbol{\theta}} f(\mathbf{X}, \boldsymbol{\theta}') \in \mathbb{R}^{d_\text{out}|\mathcal{D}| \times d_{\boldsymbol{\theta}}} \). Under the reparameterization, \( f^\text{lin}(\cdot, \boldsymbol{\theta}) \) becomes \( g^\text{lin}(\cdot, \Delta \boldsymbol{\alpha}) \coloneq f^\text{lin}(\cdot, \hat{\boldsymbol{\theta}}^\star) + \mathbf{K}(\cdot, \mathbf{X} )\Delta \boldsymbol{\alpha} \) and \( \hat{\mathcal{L}}_{\mathcal{D}}(\boldsymbol{\theta}) \) becomes
	\begin{equation*}
			\tilde{\mathcal{L}}_{\mathcal{D}}(\Delta\boldsymbol{\alpha}) \coloneq \frac{1}{|\mathcal{D}|} \sum_{(\mathbf{x}, \mathbf{y}) \in \mathcal{D}} \Big ( \ell(g^\text{lin}(\mathbf{x}, \Delta \boldsymbol{\alpha}), \mathbf{y}) + \frac{\lambda}{2} \| \phi(\Delta \boldsymbol{\alpha}) - \boldsymbol{\theta}' \|_2^2 \Big ),
		\end{equation*}
	where \( \mathbf{K}(\cdot, \cdot) \coloneq \nabla_{\boldsymbol{\theta}} f(\cdot, \boldsymbol{\theta}') \nabla_{\boldsymbol{\theta}} f(\cdot, \boldsymbol{\theta}')^\top \) denotes the empirical Neural Tangent Kernel (NTK) \citep{jacot2018neural}. The optimization problem \( \min_{\boldsymbol{\Delta \alpha} \in \mathbb{R}^{d_\text{out} |\mathcal{D}|}} \tilde{\mathcal{L}}_{\mathcal{D}}(\Delta \boldsymbol{\alpha}) \) is an equivalent dual formulation of \( \min_{\boldsymbol{\theta} \in \mathbb{R}^{d_{\boldsymbol{\theta}}}} \hat{\mathcal{L}}_{\mathcal{D}}(\boldsymbol{\theta}) \). The following theorem characterizes their relationship.
	
	
	\begin{theorem} \label{thm: dual representation}
		If \( \ell \) is convex with respect to model output, \( \lambda > 0 \), and \( \nabla_{\boldsymbol{\theta}}f\left(\mathbf{X}, \boldsymbol{\theta}'\right) \) has full row rank, then for any nonempty \( \mathcal{D}_r \subseteq \mathcal{D} \), the optimization problems: \(\min_{\boldsymbol{\Delta \alpha} \in \mathbb{R}^{d_\text{out} |\mathcal{D}|}} \tilde{\mathcal{L}}_{\mathcal{D}_r}(\Delta \boldsymbol{\alpha}) \) and \( \min_{\boldsymbol{\theta} \in \mathbb{R}^{d_{\boldsymbol{\theta}}}} \hat{\mathcal{L}}_{\mathcal{D}_r}(\boldsymbol{\theta}) \), have unique minimizers \( \boldsymbol{\Delta \alpha}_r^\star \) and \( \hat{\boldsymbol{\theta}}_r^\star \) that are related by \( \boldsymbol{\Delta \alpha}_r^\star \xrightleftharpoons[\phi^{-1}(\cdot)]{\phi(\cdot)} \hat{\boldsymbol{\theta}}_r^\star \). Moreover, \( \boldsymbol{\Delta \alpha}_r^\star \) admits the form
		\begin{equation*}
			\Delta \boldsymbol{\alpha}_r^\star \coloneq
			\begin{bmatrix}
				\frac{1}{\lambda} \nabla_{f^\text{lin}(\mathbf{X}_f, \hat{\boldsymbol{\theta}}^\star)} \hat{\mathcal{L}}_{\mathcal{D}} \\
				-\frac{1}{\lambda} \left( \nabla_{f^\text{lin}(\mathbf{X}_r, \hat{\boldsymbol{\theta}}_r^\star)} \hat{\mathcal{L}}_{\mathcal{D}_r} - \nabla_{f^\text{lin}(\mathbf{X}_r, \hat{\boldsymbol{\theta}}^\star)} \hat{\mathcal{L}}_{\mathcal{D}} \right) 
			\end{bmatrix}.
		\end{equation*}
	\end{theorem}
	
	\subsection{\texorpdfstring{Influence on \( \Delta \boldsymbol{\alpha} \)}{Influence on Delta alpha}} \label{subsec:Influence on Delta alpha}
	
	The dual formulation provides an alternative way to obtain \( \hat{\boldsymbol{\theta}}^\star \). Specifically, \( \hat{\boldsymbol{\theta}}^\star \) can be obtained by first obtaining \( \Delta \boldsymbol{\alpha}^\star = \arg \ \min_{\boldsymbol{\Delta \alpha} \in \mathbb{R}^{d_\text{out} |\mathcal{D}|}} \tilde{\mathcal{L}}_{\mathcal{D}}(\Delta \boldsymbol{\alpha}) \), then mapping \( \Delta \boldsymbol{\alpha}^\star \) back to \( \hat{\boldsymbol{\theta}}^\star \) via \( \phi(\cdot) \). Therefore, to study the influence of removing a subset of data points from \( \mathcal{D} \) on \( \hat{\boldsymbol{\theta}}^\star \), it suffices to study the influence on \( \Delta \boldsymbol{\alpha}^\star \). For any nonempty \( \mathcal{D}_r \subseteq \mathcal{D} \), define \( \Delta \boldsymbol{\alpha}_r^\star \coloneq \arg \ \min_{\boldsymbol{\Delta \alpha} \in \mathbb{R}^{d_\text{out} |\mathcal{D}|}} \tilde{\mathcal{L}}_{\mathcal{D}_r}(\Delta \boldsymbol{\alpha}) \), by following the same derivation as in Section \ref{subsec:Influence Functions}, \( \Delta \boldsymbol{\alpha}_r^\star \) can be approximated by
	\begin{equation}
		\Delta \boldsymbol{\alpha}_r^\star \approx \frac{|\mathcal{D}_f|}{|\mathcal{D}|} \mathbf{H}_{\Delta \boldsymbol{\alpha}^\star}^{-1} \nabla_{\Delta \boldsymbol{\alpha}} \tilde{\mathcal{L}}_{\mathcal{D}_f}(\mathbf{0}),
		\label{eq:influence on delta alpha}
	\end{equation}
	where \( \mathbf{H}_{\Delta \boldsymbol{\alpha}^\star} \coloneqq \nabla_{\Delta \boldsymbol{\alpha}}^2 \tilde{\mathcal{L}}_\mathcal{D}(\mathbf{0}, -\frac{1}{|\mathcal{D}|}) \in \mathbb{R}^{d_\text{out}|\mathcal{D}| \times d_\text{out}|\mathcal{D}|}\) and \( \mathbf{H}_{\Delta \boldsymbol{\alpha}^\star}^{-1} \) may be further approximated by \( \nabla_{\Delta \boldsymbol{\alpha}}^2 \tilde{\mathcal{L}}_\mathcal{D}(\mathbf{0})^{-1} \) when \( \frac{|\mathcal{D}_f|}{|\mathcal{D}|} \) is small. Furthermore, let \( \mathbf{K} \), \( \mathbf{K}_r \) and \( \mathbf{K}_f \) denote \( \mathbf{K}(\mathbf{X}, \mathbf{X}) \), \( \mathbf{K}(\mathbf{X}_r, \mathbf{X}) \) and \( \mathbf{K}(\mathbf{X}_f, \mathbf{X}) \), respectively, the \( \mathbf{H}_{\Delta \boldsymbol{\alpha}^\star} \) and  \( \nabla_{\Delta \boldsymbol{\alpha}} \tilde{\mathcal{L}}_{\mathcal{D}_f}(\mathbf{0}) \) in \eqref{eq:influence on delta alpha} can be expressed as
	\begin{align}
		&\mathbf{H}_{\Delta \boldsymbol{\alpha}^\star} = \frac{|\mathcal{D}_r|}{|\mathcal{D}|} \left( \mathbf{K}_r^\top \nabla_{f^\text{lin}(\mathbf{X}_r, \hat{\boldsymbol{\theta}}^\star)}^2 \hat{\mathcal{L}}_{\mathcal{D}_r} \ \mathbf{K}_r + \lambda \mathbf{K} \right) \label{eq:Hessian with repsect to delta alpha star} \\
		&\nabla_{\Delta \boldsymbol{\alpha}} \tilde{\mathcal{L}}_{\mathcal{D}_f}(\mathbf{0}) = \mathbf{K}_f^\top \nabla_{f^\text{lin}(\mathbf{X}_f, \hat{\boldsymbol{\theta}}^\star)} \hat{\mathcal{L}}_{\mathcal{D}_f} + \lambda \mathbf{K} \boldsymbol{\alpha}^\star, \label{eq:grad of reparameterized loss at 0}
	\end{align}
	where \( \boldsymbol{\alpha}^\star \coloneq -\frac{1}{\lambda}\nabla_{f^\text{lin}(\mathbf{X}, \hat{\boldsymbol{\theta}}^\star)} \hat{\mathcal{L}}_{\mathcal{D}} \in \mathbb{R}^{d_\text{out}|\mathcal{D}|} \) and \( \nabla_{f^\text{lin}(\mathbf{X}_r, \hat{\boldsymbol{\theta}}^\star)}^2 \hat{\mathcal{L}}_{\mathcal{D}_r} \in \mathbb{R}^{d_\text{out}|\mathcal{D}_r| \times d_\text{out}|\mathcal{D}_r|} \). After obtaining \(  \Delta \boldsymbol{\alpha}_r^\star \), the change in model output and loss at a test data point \( \mathbf{z}_t \) can be approximated by
	\begin{align}
		&f^\text{lin}(\mathbf{x}_t, \hat{\boldsymbol{\theta}}_r^\star) - f^\text{lin}(\mathbf{x}_t, \hat{\boldsymbol{\theta}}^\star) \approx \mathbf{K}(\mathbf{x}_t, \mathbf{X})\Delta \boldsymbol{\alpha}_r^\star \label{eq:change in output delta alpha} \\
		&\hat{\mathcal{L}}_{\{ \mathbf{z}_t \}}(\hat{\boldsymbol{\theta}}_r^\star) - \hat{\mathcal{L}}_{\{ \mathbf{z}_t \}}(\hat{\boldsymbol{\theta}}^\star) \approx \nabla_{\Delta \boldsymbol{\alpha}} \tilde{\mathcal{L}}_{\{ \mathbf{z}_t \}}(\mathbf{0})^\top \Delta \boldsymbol{\alpha}_r^\star, \label{eq:change in loss delta alpha}
	\end{align}
	where \( \nabla_{\Delta \boldsymbol{\alpha}} \tilde{\mathcal{L}}_{\{ \mathbf{z}_t \}}(\mathbf{0}) \) is given by the same expression as \eqref{eq:grad of reparameterized loss at 0}, after replacing \( \mathcal{D}_f \) with \( \{\mathbf{z}_t\} \) and \( \mathbf{X}_f \) with \( \mathbf{x}_t \). Moreover, \eqref{eq:change in output delta alpha} and \eqref{eq:change in loss delta alpha} can be evaluated at multiple test data points simultaneously via vectorization. Let \( \mathcal{D}_t \) denote a test dataset, the change in model output at all data points in \( \mathcal{D}_t \) can be computed by \( \mathbf{K}_t \Delta \boldsymbol{\alpha}_r^\star \), and the change in loss at all data points in \( \mathcal{D}_t \) can be computed by
	\begin{equation}
		\label{eq:change in loss vectorization}
		|\mathcal{D}_t| \nabla_{f^\text{lin}(\mathbf{X}_t, \hat{\boldsymbol{\theta}}^\star)} \hat{\mathcal{L}}_{\mathcal{D}_t} \odot \left( \mathbf{K}_t \ \Delta \boldsymbol{\alpha}_r^\star \right) + \lambda \boldsymbol{\alpha}^{\star\top} \mathbf{K} \ \Delta \boldsymbol{\alpha}_r^\star \mathbf{1},
	\end{equation}
	where \( \mathbf{K}_t \) denotes \( \mathbf{K}(\mathbf{X}_t, \mathbf{X}) \) and \( \mathbf{1} \in \mathbb{R}^{d_\text{out}|\mathcal{D}_t|} \) denotes the all-ones vector \footnote{Equation \eqref{eq:change in loss vectorization} can be used directly for scalar-valued model, i.e., \( d_\text{out} = 1 \). For vector-valued model, one needs to sum every \( d_\text{out} \) consecutive rows of the first term and change the dimension of \( \mathbf{1} \) to \( \mathbb{R}^{|\mathcal{D}_t|} \).}.
	
	Equation \eqref{eq:influence on delta alpha}, \eqref{eq:change in output delta alpha}, and \eqref{eq:change in loss delta alpha} are the \( \Delta \boldsymbol{\alpha} \)-space counterparts of \eqref{eq:influence on theta}, \eqref{eq:change in output theta} and \eqref{eq:change in loss theta}. Although working in the \( \Delta \boldsymbol{\alpha} \)-space may appear less convenient due to the additional step of mapping back to the \( \boldsymbol{\theta} \)-space, it nonetheless offers two advantages over working directly in \( \boldsymbol{\theta} \)-space.
	
	\subsection{A Reduced Linear System} \label{subsec:A Reduced Linear System}
	
	It can be seen from Theorem \ref{thm: dual representation} that, entries of \( \Delta \boldsymbol{\alpha}_r^\star \) corresponding to data points in \( \mathcal{D}_f \), i.e., \( \Delta \boldsymbol{\alpha}_{r,f}^\star \coloneq \frac{1}{\lambda} \nabla_{f^\text{lin}(\mathbf{X}_f, \hat{\boldsymbol{\theta}}^\star)} \hat{\mathcal{L}}_{\mathcal{D}} \), are known once \( \hat{\boldsymbol{\theta}}^\star \) is known. This gives rise to a computational shortcut for evaluating \eqref{eq:influence on delta alpha}. Let \( \mathbf{K}_{ij}\) denotes \( \mathbf{K}(\mathbf{X}_i, \mathbf{X}_j) \), \eqref{eq:influence on delta alpha} can be expressed in block matrix form as
	\begin{equation*}
		\begin{bmatrix}
			\mathbf{H}_{\Delta \boldsymbol{\alpha}^\star}^{ff} & \mathbf{H}_{\Delta \boldsymbol{\alpha}^\star}^{fr} \\
			\mathbf{H}_{\Delta \boldsymbol{\alpha}^\star}^{rf} & \mathbf{H}_{\Delta \boldsymbol{\alpha}^\star}^{rr} \\
		\end{bmatrix}
		\begin{bmatrix}
			\Delta \boldsymbol{\alpha}_{r,f}^\star  \\
			\Delta \boldsymbol{\alpha}_{r,r}^\star 
		\end{bmatrix}
		= \frac{|\mathcal{D}_f|}{|\mathcal{D}|}
		\begin{bmatrix}
			\nabla_{\Delta \boldsymbol{\alpha}} \tilde{\mathcal{L}}_{\mathcal{D}_f}(\mathbf{0})_f \\
			\nabla_{\Delta \boldsymbol{\alpha}} \tilde{\mathcal{L}}_{\mathcal{D}_f}(\mathbf{0})_r
		\end{bmatrix},
	\end{equation*}
	where
	\begin{align}
		&\mathbf{H}_{\Delta \boldsymbol{\alpha}^\star}^{ij} = \frac{|\mathcal{D}_r|}{|\mathcal{D}|} \left( \mathbf{K}_{ir} \nabla_{f^\text{lin}(\mathbf{X}_r, \hat{\boldsymbol{\theta}}^\star)}^2 \hat{\mathcal{L}}_{\mathcal{D}_r} \ \mathbf{K}_{rj} + \lambda \mathbf{K}_{ij} \right) \label{eq:H^ij} \\
		&\nabla_{\Delta \boldsymbol{\alpha}} \tilde{\mathcal{L}}_{\mathcal{D}_f}(\mathbf{0})_i = \mathbf{K}_{if} \nabla_{f^\text{lin}(\mathbf{X}_f, \hat{\boldsymbol{\theta}}^\star)} \hat{\mathcal{L}}_{\mathcal{D}_f} + \lambda \mathbf{K}_i \boldsymbol{\alpha}^\star. \label{eq:grad L_Df(0)_i}
	\end{align}
	Since \( \Delta \boldsymbol{\alpha}_{r,f}^\star \) is known and \( \mathbf{H}_{\Delta \boldsymbol{\alpha}^\star}^{rr} \in \mathbb{R}^{d_\text{out}|\mathcal{D}_r| \times d_\text{out}|\mathcal{D}_r|} \) is positive definite when \( \mathbf{H}_{\Delta \boldsymbol{\alpha}^\star} \) is positive definite, to find \( \Delta \boldsymbol{\alpha}_r^\star \), it suffices to find \( \Delta \boldsymbol{\alpha}_{r,r}^\star \) by solving
	\begin{equation}
		\label{eq:reduced linear system} 
		\mathbf{H}_{\Delta \boldsymbol{\alpha}^\star}^{rr} \Delta \boldsymbol{\alpha}_{r,r}^\star = \frac{|\mathcal{D}_f|}{|\mathcal{D}|} \nabla_{\Delta \boldsymbol{\alpha}} \tilde{\mathcal{L}}_{\mathcal{D}_f}(\mathbf{0})_r - \mathbf{H}_{\Delta \boldsymbol{\alpha}^\star}^{rf} \Delta \boldsymbol{\alpha}_{r,f}^\star.
	\end{equation}
	The linear system in  \eqref{eq:reduced linear system} has two key properties: (i) the dimension of \( \mathbf{H}_{\Delta \boldsymbol{\alpha}^\star}^{rr} \) grows with \( d_\text{out}|\mathcal{D}_r| \) instead of \( d_{\boldsymbol{\theta}} \), and (ii) the size of the system decreases as \( |\mathcal{D}_f| \) increases. Therefore, depending on the values of \( d_\text{out} \) and \( |\mathcal{D}_f| \), the system can be solved either by forming \( \mathbf{H}_{\Delta \boldsymbol{\alpha}^\star}^{rr} \) explicitly and using a standard linear system solver, or by applying CG with an operator that evaluates \( \mathbf{H}_{\Delta \boldsymbol{\alpha}^\star}^{rr} \mathbf{v} \). In either case, the computational cost of solving the system is independent of model size, and a larger \( |\mathcal{D}_f| \) yields a smaller system, reducing the cost even further.
	
	\subsection{Kernel Trick}  \label{subsec:Kernel Trick}
	
	The \( \Delta \boldsymbol{\alpha} \)-space representation of influence functions allows influence functions to be evaluated using the kernel trick. This follows from observing that \( \tilde{\mathcal{L}}_{\mathcal{D}}(\Delta\boldsymbol{\alpha}) \) is a special case of the more general empirical risk
	\begin{equation*}
		\mathcal{L}_{\mathcal{D}}(\boldsymbol{\alpha}) \coloneq \frac{1}{|\mathcal{D}|} \sum_{(\mathbf{x}, \mathbf{y}) \in \mathcal{D}} \Big ( \ell(h(\mathbf{x}, \boldsymbol{\alpha}), \mathbf{y}) + \frac{\lambda}{2} \boldsymbol{\alpha}^\top \boldsymbol{\Theta}(\mathbf{X}, \mathbf{X}) \boldsymbol{\alpha} \Big ),
	\end{equation*}
	where \( h(\cdot, \boldsymbol{\alpha}) \coloneq f'(\cdot) + \boldsymbol{\Theta}(\cdot, \mathbf{X})\boldsymbol{\alpha} \) and \( \boldsymbol{\Theta}(\mathbf{X}, \mathbf{X}) \succ 0 \) is the Gram matrix induced by kernel function \( \boldsymbol{\Theta}(\cdot, \cdot) \) on \( \mathcal{D} \). Specifically, if \( f'(\cdot) = f(\cdot, \boldsymbol{\theta}') \) and \( \boldsymbol{\Theta}(\cdot, \cdot) = \mathbf{K}(\cdot, \cdot) \), then \( \tilde{\mathcal{L}}_{\mathcal{D}}(\Delta\boldsymbol{\alpha}) = \mathcal{L}_{\mathcal{D}}(\boldsymbol{\alpha}^\star + \Delta \boldsymbol{\alpha}) \), where \( \boldsymbol{\alpha}^\star \coloneq \arg \ \min_{\boldsymbol{\alpha} \in \mathbb{R}^{d_\text{out} |\mathcal{D}|}} \mathcal{L}_{\mathcal{D}}( \boldsymbol{\alpha}) \). Therefore, by replacing \( f'(\cdot) \) and  \( \boldsymbol{\Theta}(\cdot, \cdot) \) with suitable alternatives and following the same derivation as the original influence functions, the influence of removing data points can be computed for a broad class of models of the form \( h(\cdot,\boldsymbol{\alpha}) \) trained by minimizing \( \mathcal{L}_{\mathcal{D}}(\boldsymbol{\alpha}^\star + \Delta \boldsymbol{\alpha}) \), or equivalently, \( \mathcal{L}_{\mathcal{D}}(\boldsymbol{\alpha}) \). This also includes models whose \( \boldsymbol{\theta} \)-space representation contains infinitely many parameters. For example, by replacing \( \boldsymbol{\Theta}(\cdot, \cdot) \) with the analytical NTK \( \mathbf{K}^\infty(\cdot, \cdot) \) \citep{jacot2018neural,lee2019wide}, we can compute the influence of removing data points on the loss and model outputs of infinitely wide neural networks whose NTKs admit analytical forms.
	
	\subsection{Discussion} \label{subsec:Discussion}
	
	The estimates provide by the original representation of influence functions are most reliable when the empirical risk is strongly convex with respect to the model parameters, and the model is trained to convergence, i.e.,  \( \nabla_{\boldsymbol{\theta}}\mathcal{L}_{\mathcal{D}}(\boldsymbol{\theta}^\star) = \mathbf{0} \). In addition to the conditions stated in Theorem \ref{thm: dual representation}, the reliable use of the dual representation requires \( \nabla_{\Delta \boldsymbol{\alpha}} \tilde{\mathcal{L}}_{\mathcal{D}}(\Delta \boldsymbol{\alpha}^\star) = \mathbf{0} \), which can be achieved either by training \( g^\text{lin}(\cdot, \Delta \boldsymbol{\alpha}) \) to convergence, or by training \( f^\text{lin}(\cdot, \boldsymbol{\theta}) \) to convergence. 
	
	
	Although evaluating \eqref{eq:influence on delta alpha} is more computationally efficient than evaluating \eqref{eq:influence on theta} when \( d_\text{out} |\mathcal{D}_r| \ll d_{\boldsymbol{\theta}} \ \), the evaluations of \eqref{eq:change in output delta alpha} and \eqref{eq:change in loss delta alpha} do not offer an obvious computational advantage over those of \eqref{eq:change in output theta} and \eqref{eq:change in loss theta} when the NTK does not admit an analytical form. This may appear to limit the computational advantage of the dual representation for finite models, where the empirical NTK matrix must be constructed from per-example gradients. However, this concern is mitigated by the fact that, once \( \Delta \boldsymbol{\alpha}_r^\star \) is obtained, it can be mapped back to \( \hat{\boldsymbol{\theta}}_r^\star \) through \( \phi(\cdot) \), and \eqref{eq:change in output theta} and \eqref{eq:change in loss theta} can then be used to study the influence on model output and loss. Compared with the approach that first compute \( \hat{\boldsymbol{\theta}}_r^\star \) using \eqref{eq:influence on theta} and then applies  \eqref{eq:change in output theta} and \eqref{eq:change in loss theta}, our approach is more efficient as long as the combined cost of evaluting \eqref{eq:influence on delta alpha} and applying the map \( \phi(\cdot) \) is lower than the cost of evaluting \eqref{eq:influence on theta}. The experiment results shown in Figure \ref{fig:figure_1}, where we compare the time it takes to obtain an estimate of \( \hat{\boldsymbol{\theta}}_r^\star \), using the two approaches, confirm that this combined cost is indeed lower in the settings considered.
	
	Apart from linear and linearized models, the dual representation of influence functions can also be applied to nonlinear models, provided that they are linearizable. Informally, we say a parameterized model is linearizable at some reference point \( \boldsymbol{\theta}' \) if the model and its linearization around \( \boldsymbol{\theta}' \) produce similar training dynamics. Nonlinear models can be made approximately linearizable by training with \( \hat{\mathcal{L}}_\mathcal{D} \) using large values of \( \lambda \). This is because a nonlinear model is well approximated by its linearization around \( \boldsymbol{\theta}' \) when its parameters change little from \( \boldsymbol{\theta}' \) during training, and the regularization term \( \lambda \| \boldsymbol{\theta} - \boldsymbol{\theta}' \|_2^2 \) discourages the parameters from deviating too far from \( \boldsymbol{\theta}' \)  when \( \lambda \) is large. Therefore, when trained with large \( \lambda \), the above results may still hold even if the linearized model \( f^\text{lin}(\cdot, \boldsymbol{\theta}) \) in \( \hat{\mathcal{L}}_\mathcal{D} \) is replaced by its nonlinear counterpart \( f(\cdot, \boldsymbol{\theta}) \). 
	
	
	
	\section{Experiments} \label{sec:Experiments}
	
	This section presents empirical verification of the advantages claimed in Sections \ref{subsec:A Reduced Linear System} and \ref{subsec:Kernel Trick} for working in \( \Delta \boldsymbol{\alpha} \)-space and investigates how the choice of \( \lambda \) affects the training dynamics of nonlinear models and the effectiveness of the two representations of influence functions. All experiments are implemented in Python using JAX \citep{jax2018github} and Neural Tangents \citep{neuraltangents2020, novak2022fast, hron2020infinite, sohl2020infinite, han2022fast} libraries, and all results are obtained on two NVIDIA RTX 3080 Ti GPUs. The code to reproduce the experiments in this section can be found in \url{https://github.com/ZhenhuanSun/ntk_influence}.
	
	\subsection{Implementation Details}
	
	The original and dual representations of influence functions are evaluated by solving \eqref{eq:linear system in theta} and \eqref{eq:reduced linear system}, respectively, using CG. The operator \( \mathbf{v} \mapsto \mathbf{H}_{\boldsymbol{\theta}^\star} \mathbf{v} \) is implemented via automatic differentiation, while the operator \( \mathbf{v} \mapsto \mathbf{H}_{\Delta \boldsymbol{\alpha}^\star}^{rr} \mathbf{v} \) is implemented based on \eqref{eq:H^ij}, where the HVP \( \nabla_{f^\text{lin}(\mathbf{X}_r, \hat{\boldsymbol{\theta}}^\star)}^2 \hat{\mathcal{L}}_{\mathcal{D}_r} \mathbf{v} \) is computed using automatic differentiation and \( \mathbf{K}_{rr} \mathbf{v} \) is computed as a matrix vector product. 
	
	We assume that \( \mathbf{K} \) is precomputed and stored, so that \( \mathbf{K}_{rr} \) can be obtained by slicing \( \mathbf{K} \).  In practice, constructing \( \mathbf{K} \) in advance may be advantageous, since, once it is constructed, it can be reused across different influence function calls to study the influence of removing different subsets of data points. This is because, no matter which data points we remove, the components of \eqref{eq:reduced linear system} can always be expressed as submatrices of \( \mathbf{K} \), and these submatrices can be obtained by simply slicing \( \mathbf{K} \). In practice, \( \mathbf{K}_{rr} \) may be stored either on a single GPU or sharded across multiple GPUs to reduce per-device memory usage. In the latter case, we implement a sharded version of the matrix vector product \( \mathbf{K}_{rr} \mathbf{v} \) that allows \( \mathbf{K}_{rr} \mathbf{v} \) to be evaluated while keeping the shards of \( \mathbf{K}_{rr} \) distributed across multiple devices.
	
	
	
	\subsection{Linearized Models}
	
	We verify the computational advantage of the dual representation of influence functions, hereafter referred to as the \( \Delta \boldsymbol{\alpha} \)-space method, by comparing its unlearning performance with the the original influence functions, hereafter referred to as the \( \boldsymbol{\theta} \)-space method, on two linearized neural networks and two datasets. For each model-dataset pair, we train the model on full dataset, remove \( p \in \{ 10, 30, 50, 70, 90 \} \) percent of the data points randomly from either all classes or a specific class of the dataset, and apply both methods for unlearning. The unlearning performance of both methods is evaluated using the following metrics:
	\begin{enumerate}[leftmargin=*]
		\item \textbf{Cold-start runtime}, defined as the time required to run each method on its first execution. This includes not only the time to prepare and solve the linear systems, but also one-time overhead such as compilation, memory allocation, and other initialization costs. This metric compares unlearning speed in scenarios where unlearning is required only once. To accurately measure this runtime, for each value of \( p \), the experiment is run in a separate process to prevent reuse of compiled functions and cached resources;
		\item \textbf{Warm-start runtime}, defined as the mean and standard deviation of the runtime over five subsequent executions after initialization overhead has been incurred. This metric compares unlearning speed in scenarios where unlearning needs to be performed repeatedly on different \( \mathcal{D}_f \) of the same size;
		\item \textbf{Relative $\ell_2$ distance} between the unlearned model parameters \( \boldsymbol{\theta}_u \) and the retrained model parameters \( \hat{\boldsymbol{\theta}}_r^\star \), which measures how closely the unlearned model aligns with the retrained model in parameter space. A smaller value indicates better alignment between them;
		\item \textbf{Accuracy} of the unlearned models and the retrained model on \( \mathcal{D}_f \), which is used to assess forgetting quality. The closer the accuracy of the unlearned model is to that of the retrained model, the stronger the indication that the unlearned model has forgotten the designated data points to a degree similar to the retrained model.
	\end{enumerate}
	
	We consider the linearizations of two models: a fully connected neural network (FCNN) with three hidden layers of width \( 1024 \) and ReLU activations, and a convolutional neural network (CNN) composed of three blocks, each consisting of a convolutional layer with \( 128 \) channels and a \( 3 \times 3 \) kernel, followed by a ReLU activation and average pooling. The FCNN and CNN are both linearized around their initialization and trained, respectively, on a subset of MNIST \citep{lecun2002gradient} consisting of \( 200 \) data points from each of the 10 classes and on a subset of CIFAR-10 \citep{krizhevsky2009learning} consisting of \( 1000 \) data points from each of classes 0 (airplane) and 1 (automobile). For FCNN-MNIST pair, training is performed by minimizing \( \hat{\mathcal{L}}_\mathcal{D} \) with cross entropy loss and \( \lambda = 10^{-2} \) using full-batch gradient descent, and data points are randomly removed from all classes of the constructed dataset. For the CNN-CIFAR-10 pair, training is performed by minimizing \( \hat{\mathcal{L}}_\mathcal{D} \) with squared error loss and \( \lambda = 10^{-1} \) using momentum, and data points are removed from class \( 0 \) of the constructed dataset. To highlight the effect of unlearning, we include a random perturbation baseline that perturbs the trained model parameter \( \hat{\boldsymbol{\theta}}^\star \) with random noise \( \boldsymbol{\epsilon} \in \mathbb{R}^{d_{\boldsymbol{\theta}}} \), where \( \boldsymbol{\epsilon} \) satisfies \( \| \boldsymbol{\epsilon} \|_2 = \| \hat{\boldsymbol{\theta}}^\star - \hat{\boldsymbol{\theta}}_r^\star \|_2 \). The unlearning performance of the \( \boldsymbol{\theta} \)-space and \( \Delta \boldsymbol{\alpha} \)-space methods on these models is shown in Figure \ref{fig:figure_1}.
	
	\begin{figure}[t]
		\centering
		\vspace{-\baselineskip}
		
		\includegraphics[width=0.9\textwidth]{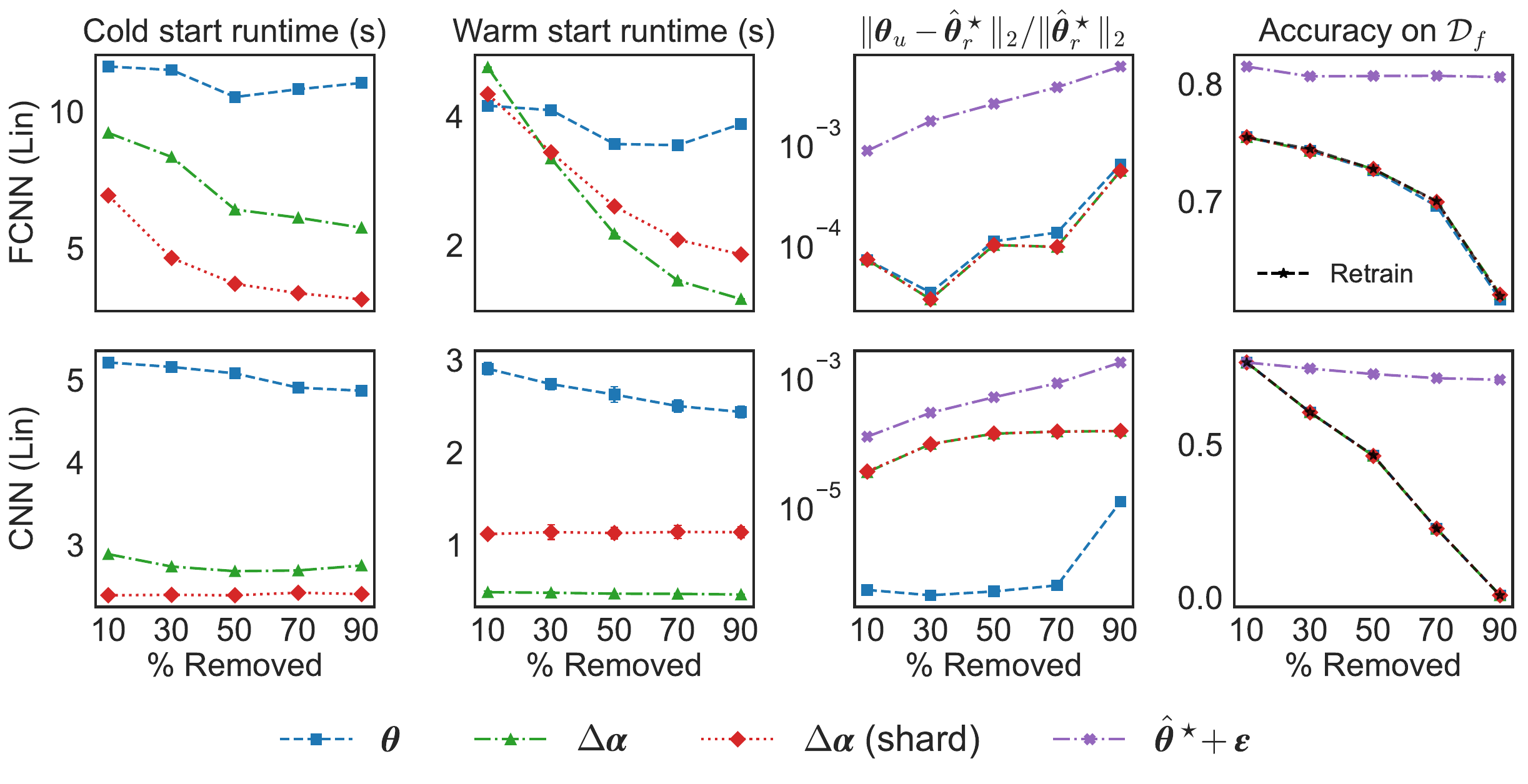}
		\vspace{-0.5\baselineskip}
		\caption{\textbf{Comparison of the unlearning performance of \( \boldsymbol{\theta} \)-space and \( \Delta \boldsymbol{\alpha} \)-space methods on linearized FCNN and CNN trained on subsets of MNIST and CIFAR10.} The top and bottom rows correspond to the linearized FCNN and CNN, respectively. We compare the \( \boldsymbol{\theta} \)-space method with two implementations of the \( \Delta \boldsymbol{\alpha} \)-space method, and with a random perturbation baseline. From left to right, we report cold-start runtime, warm-start runtime, relative parameter distance, and test accuracy as a function of the percentage of removed data. The dashed black curves in the rightmost panes indicate the accuracy of the retrained model. We observe that, when model size is large relative to dataset size, the \( \Delta \boldsymbol{\alpha} \)-space method requires less runtime while maintaining  unlearning performance comparable to that of the \( \boldsymbol{\theta} \)-space method.}
		\label{fig:figure_1}
		\vspace{-\baselineskip}
	\end{figure}
	
	In both experiments, the two implementations of the \( \Delta \boldsymbol{\alpha} \)-space method generally achieve lower runtimes than the \( \boldsymbol{\theta} \)-space method. These results are consistent with our expectations, since the dataset size in both experiments are small relative to the model size, i.e., \( d_\text{out}|\mathcal{D}| \ll d_{\boldsymbol{\theta}} \). Consequently, the cost of solving \eqref{eq:reduced linear system} is smaller than that of solving \eqref{eq:linear system in theta}. In addition, the unlearning performance of the \( \Delta \boldsymbol{\alpha} \)-space method is comparable to that of the \( \boldsymbol{\theta} \)-space method, which can be seen by the small relative \( \ell_2 \) distance between \( \hat{\boldsymbol{\theta}}_r^\star \) and \( \boldsymbol{\theta}_u \) produced by both methods and by their close agreement with the retrained model in accuracy on \( \mathcal{D}_f \).
	
	
	\subsection{\texorpdfstring{Effect of \( \lambda \)}{Effect of lambda}}
	
	To investigate how \( \lambda \) affects the training dynamics of nonlinear models and how these dynamics compare with those of their linearizations, we simultaneously train a nonlinear model and its linearization, using the same loss and optimizer, and record their training dynamics for different values of \( \lambda \).
	
	We adopt a scalar-valued FCNN with two hidden layers of width \( 2048 \) and ReLU activations. Both the model and its linearization around initialization are trained by minimizing \( \hat{\mathcal{L}}_{\mathcal{D}}(\boldsymbol{\theta}) \) with squared error loss on a binary MNIST dataset consisting of \( 5000 \) data points from each of the classes \( 3 \) and \( 8 \), using full-batch gradient descent with a learning rate of \( 0.1 \). During training, we record the relative \( \ell_2 \) distance between the parameters of the FCNN and its linearization, denoted by \( \boldsymbol{\theta} \) and \( \boldsymbol{\theta}^\text{lin} \), respectively; the root mean squared error (RMSE) between their predictions on \( \mathcal{D}_t \); the accuracy performance of both models on \( \mathcal{D}_t \); and the gradient norms of the parameters of the FCNN and its linearization. These quantities, recorded over \( 4000 \) epochs of training, are summarized in Figure \ref{fig:figure_2}.
	
	\begin{figure}[t]
		\centering
		\includegraphics[width=0.9\textwidth]{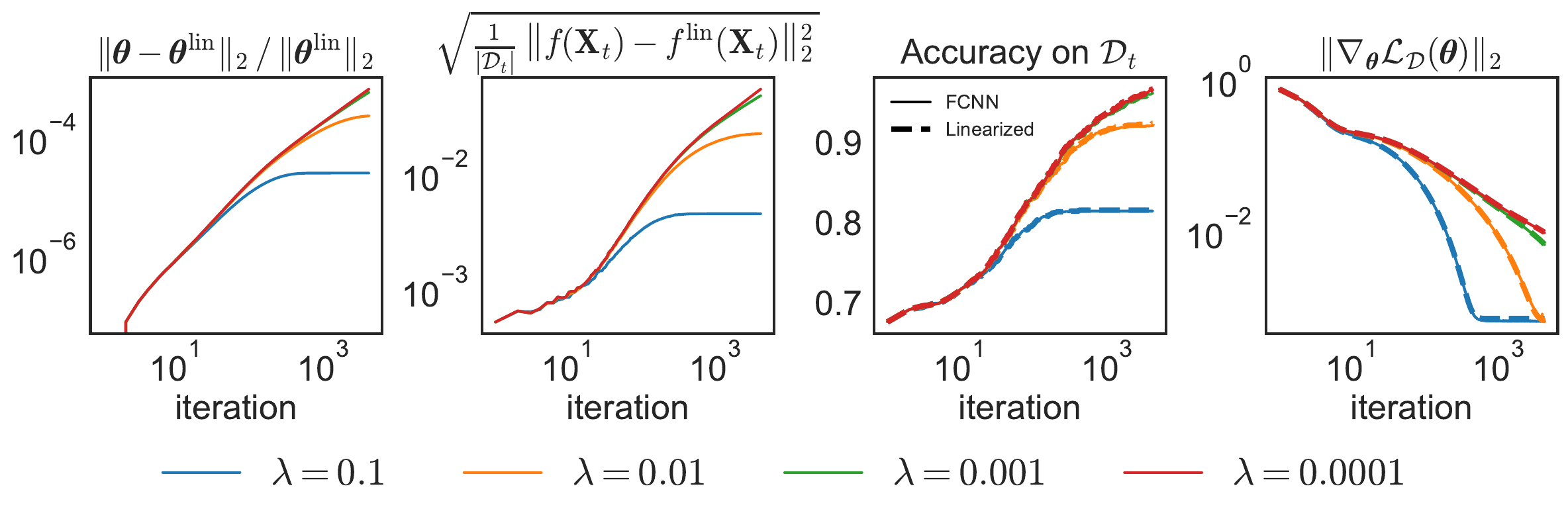}
		\vspace{-0.5\baselineskip}
		\caption{\textbf{Comparison of the training dynamics of a FCNN and its linearization under different values of \( \lambda \).} The left two panes show the relative \( \ell_2 \) distance between the parameters of FCNN and its linearization, and the RMSE between their predictions on \( \mathcal{D}_t \). The right two panes show the accuracies of the FCNN and its linearization on \( \mathcal{D}_t \), and the gradient norms of their parameters. We observe that larger values of \( \lambda \) make the FCNN track its linearization more closely and converge faster, but at the cost of worse test accuracy.}
		\label{fig:figure_2}
		\vspace{-\baselineskip}
	\end{figure}
	
	We observe that larger values of \( \lambda \) make the training dynamics of the FCNN more closely resemble that of its linearization, as evidenced by the small relative \( \ell_2 \) distance between their parameters and small difference between their outputs throughout training. In addition, \( \|\nabla_{\boldsymbol{\theta}}\mathcal{L}_{\mathcal{D}}(\boldsymbol{\theta})\|_2 \) approaches \( 0 \) more rapidly for larger values of \( \lambda \), indicating that, when trained with \( \lambda \| \boldsymbol{\theta} - \boldsymbol{\theta}' \|_2^2 \), both the model and its linearization converge faster under stronger regularization. However, this comes at the cost of worse test accuracy, as larger values of \( \lambda \) lead to progressively poorer test accuracy.
	
	Since the dual representation requires the model to be linearizable, and the original representation also benefits from working with linear models, because they make the empirical risk strongly convex in \( \boldsymbol{\theta} \), the choice of \( \lambda \) plays an important role in the effectiveness of \( \boldsymbol{\theta} \)-space and \( \Delta \boldsymbol{\alpha} \)-space methods when they are applied to nonlinear models. This sensitivity to \( \lambda \) is illustrated in Figure \ref{fig:figure_3}, where data points from both classes are unlearned from the aforementioned FCNN, trained with different values of \( \lambda \), using both methods. It can be seen that both methods produce unlearned model that more faithfully match the retrained model, when applied to the FCNN trained with larger value of \( \lambda \). This behvior is expected because,  as \( \lambda \) decreases, the model parameters leave the linearizable regime and behave increasinlgy nonlinear, violating the assumptions underlying both representations of influence functions and causing both to produce unreliable estimates.
	
	\subsection{Infinitely Wide Neural Networks}
	
	We showcase the applicability of \( \Delta \boldsymbol{\alpha} \)-space method to infinitely wide neural networks for estimating changes in model outputs and loss induced by data point removal. We consider two infinitely wide neural networks derived from a FCNN with three hidden layers and ReLU activations, and from a CNN with three convolutional layers and ReLU activations, respectively. The infinitely wide FCNN is trained by minimizing \( \mathcal{L}_{\mathcal{D}}(\boldsymbol{\alpha})  \) with cross entropy loss and \( \lambda = 10^{-1} \) on a subset of MNIST consisting of \( 1000 \) data points from each of the \( 10 \) classes, whereas the infinitely wide CNN is trained by minimizing \( \mathcal{L}_{\mathcal{D}}(\boldsymbol{\alpha})  \) with squared error loss and \( \lambda = 10^{-1} \) on a subset of CIFAR-10 consisting of \( 500 \) data points from each of classes \( 0 \) and \( 1 \). In both cases, optimization is performed using gradient descent with momentum.

	For each model-dataset pair, the model is trained on the constructed dataset and on a retain dataset obtained by randomly removing \( 50\% \) of data points from each class of the dataset, after which the actual changes in model outputs and loss due to the removal of data points are evaluated at \( 10 \) randomly selected test data points. The estimated changes in model outputs and loss are obtained by first solving \eqref{eq:reduced linear system}, and then evaluating  \( \mathbf{K}_t \Delta \boldsymbol{\alpha}_r^\star \) and \eqref{eq:change in loss vectorization} at the test data points, with \( \mathbf{K}(\cdot, \cdot) \) replaced by \( \mathbf{K}^\infty(\cdot, \cdot) \) throughout the process. A comparison of the actual and estimated differences in model outputs and loss is summarized in Figure \ref{fig:figure_4}, where, for the outputs of infinitely wide FCNN, all \( 10 \) output dimensions are shown. It can be seen that the estimates provided by \( \Delta \boldsymbol{\alpha} \)-space method are accurate for both infinitely wide neural networks.
	
	
	\begin{figure}[t]
		\centering
		
		\begin{subfigure}{0.48\textwidth}
			\centering
			\includegraphics[width=0.9\linewidth]{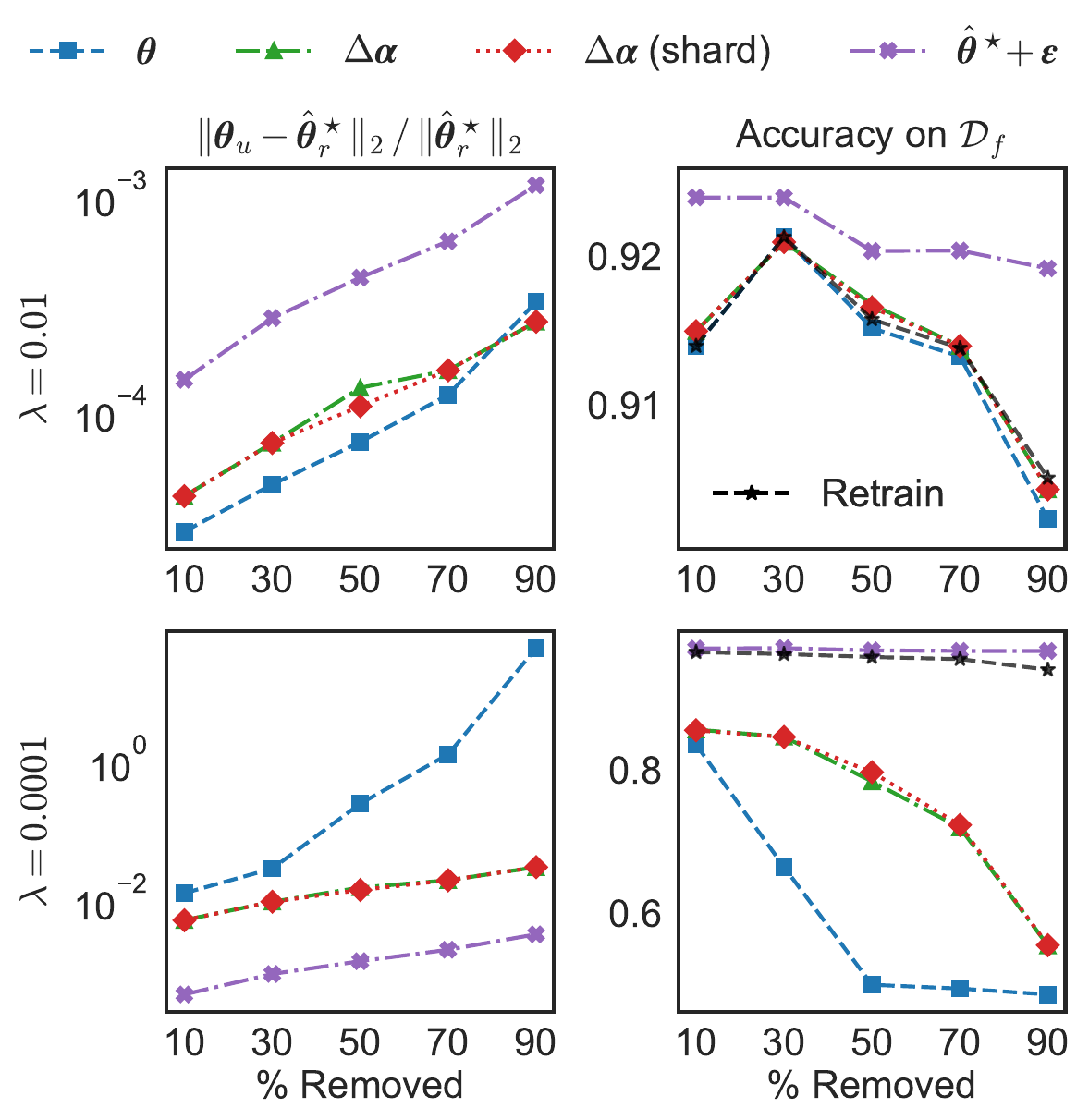}
			\vspace{-0.5\baselineskip}
			\caption{}
			\label{fig:figure_3}
		\end{subfigure}
		\hfill
		\begin{subfigure}{0.51\textwidth}
			\centering
			\includegraphics[width=0.9\linewidth]{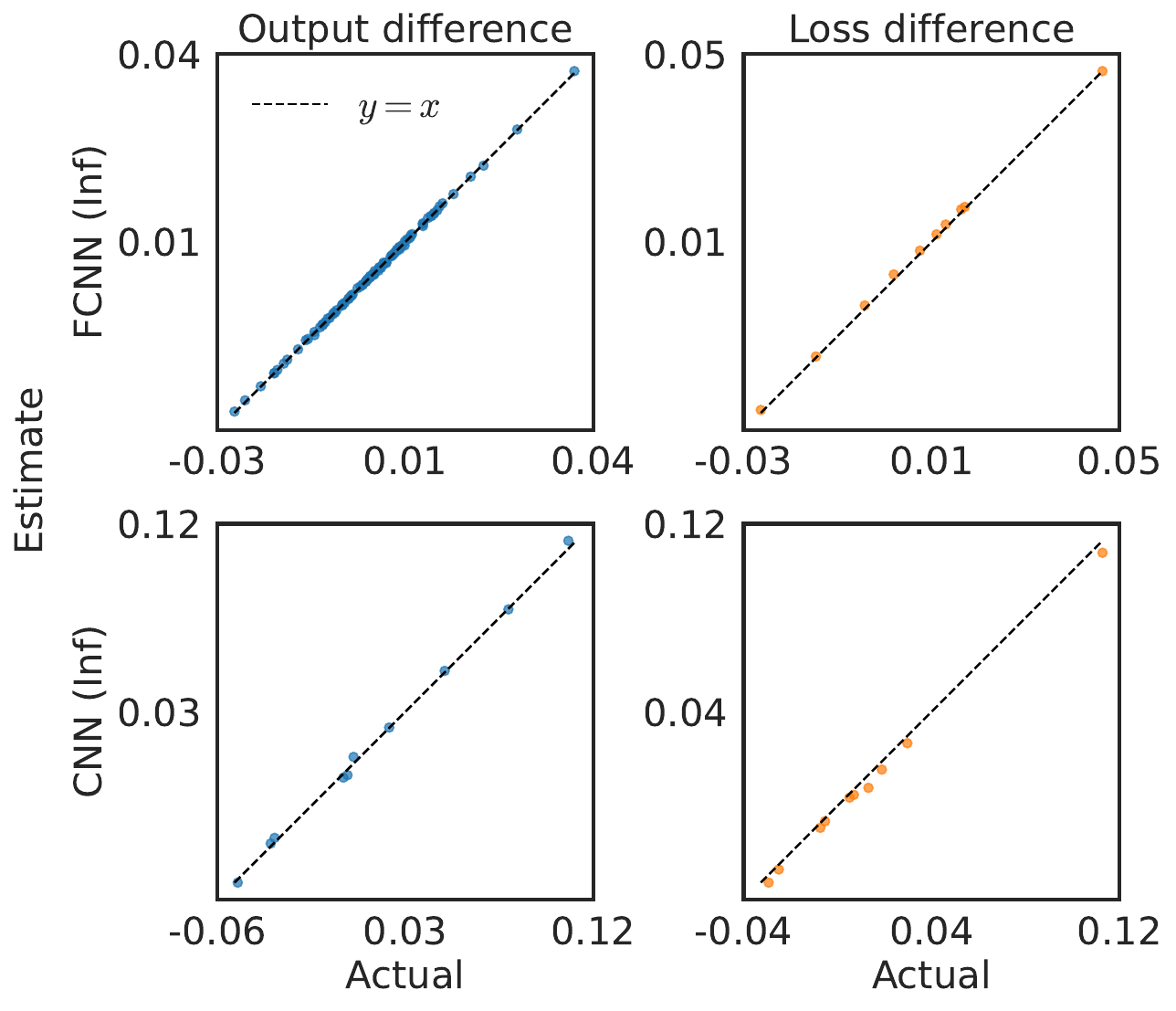}
			\vspace{-0.5\baselineskip}
			\caption{}
			\label{fig:figure_4}
		\end{subfigure}
		\vspace{-1.5\baselineskip}
		\caption{\textbf{(a) Comparison of the unlearning performance of \( \boldsymbol{\theta} \)-space and \( \Delta \boldsymbol{\alpha} \)-space methods on a nonlinear model under different values of \( \lambda \).} The dashed black curves indicate the accuracy of the retrained model. We observe that, when trained with larger value of \( \lambda \), both \( \boldsymbol{\theta} \)-space and \( \Delta \boldsymbol{\alpha} \)-space methods produce unlearned model that more faithfully match the retrained model. \textbf{(b) Comparison of estimated and actual changes in the outputs and loss of two infinitely wide neural networks at \( 10 \) test data points after removing \( 50\% \) of data points from \( \mathcal{D} \).} }
		\vspace{-\baselineskip}
	\end{figure}
	
	
	
	\section{Related Work}
	
	
	Our work builds on the work of \citet{koh2017understanding}, who applied influence functions from robust statistics to machine learning to study the effect of perturbing a training data point on model parameters, loss, and outputs. In their work, they identified the scalability issue of influence functions, namely that evaluating the IHVP using CG becomes computationally expensive as dataset size and model size increase, and consequently proposed to use a stochastic estimation method to approximate these quantities. However, \citet{basu2021influence} later observed that the stochastic estimation of the IHVP can make influence estimates unreliable in deep learning settings. Several later works focus on improving the scalability of influence functions in \( \boldsymbol{\theta} \)-space by developing more efficient methods for approximating the IHVP. For example, \citet{schioppa2022scaling} scale up influence functions by using Arnoldi iteration to approximate the IHVP, and demonstrate applicability to vision transformers with up to 300 million parameters, while \citet{grosse2023studying} employ the Eigenvalue-corrected Kronecker-Factored Approximate Curvature (EK-FAC) to approximate the IHVP, extending influence functions to large language models with up to 52 billion parameters.

	\section{Conclusion} \label{sec:Conclusion}
	
	We presented a dual representation of influence functions for linearizable models, which shifts the computational complexity of IHVP from scaling with model size to scaling with dataset size. We showed analytically that, when the number of model parameters is large relative to the product of model output dimension and retain dataset size, the dual representation allows the influence of removing data points on model parameters to be computed more efficiently than the original representation, without introducing any additional approximation. Furthermore, the dual representation makes it possible to analyze the influence on model outputs and loss in the limit where the number of model parameters goes to infinity. Experimental results obtained by comparing the dual and original representations of influence functions on different machine unlearning tasks further confirmed the computational advantages of the dual representation in settings where working directly in parameter space is difficult or infeasible. Nonetheless, the dual representation is limited in that it is applicable only to linearizable models and requires the computation and storage of the NTK matrix, which can make its use expensive for datasets with many data points and classes. 
	
	We believe that, when complemented by the dual representation, influence functions can be used to efficiently unlearn or understand the behavior of linear and linearized models regardless of model size. They may also be applicable to nonlinear models whose parameters remain close to their initialization, such as models obtained by fine tuning a pretrained model. More broadly, the dual representation may offer a different path to scalability in a literature where scalability has largely been pursued by developing increasingly sophisticated approximations to the IHVP, while also opening new research directions. One particularly interesting question to answer in the future is whether existing IHVP approximation techniques can be adapted to the dual representation to further reduce the computational cost of influence functions.
	
	
	\newpage
	
	\subsection*{AI use statement}
	
	In this work, we used generative AI tools for the following tasks:
	\begin{itemize}
		\item Assisting with writing proofs and derivations;
		\item Assisting with implementing experiments and plotting figures;
		\item Understanding unfamiliar technical concepts;
		\item Rephrasing sentences and checking grammar.
	\end{itemize}
	We did not use generative AI tools for any other tasks that require disclosure. All AI-assisted work was reviewed and verified for correctness. We take full responsibility for the final content of this work.
	
	
	\subsection*{Ethics statement}
	
	To the best of our knowledge, this work does not raise any questions regarding the Code of Ethics.
	
	
	\subsection*{Reproducibility statement}
	
	Section \ref{sec:Experiments} describes the setup for each experiment, including the GPU used to obtain the results. Additional implementation details, together with instructions and code for reproducing all experiments, are available in the GitHub repository linked in the paper. Running all experiments takes approximately \(1.5\)--\(2\) hours. All theoretical results are accompanied by proofs or derivations, which are provided in Appendices \ref{A:Proofs} and \ref{A:Derivations}, respectively.
	
	
	\bibliography{references}

@article{hampel1974influence,
  title={The influence curve and its role in robust estimation},
  author={Hampel, Frank R},
  journal={Journal of the american statistical association},
  year={1974}
}

@inproceedings{koh2017understanding,
  title={Understanding black-box predictions via influence functions},
  author={Koh, Pang Wei and Liang, Percy},
  booktitle={International Conference on Machine Learning (ICML)},
  year={2017}
}

@article{nguyen2025survey,
  title={A survey of machine unlearning},
  author={Nguyen, Thanh Tam and Huynh, Thanh Trung and Ren, Zhao and Nguyen, Phi Le and Liew, Alan Wee-Chung and Yin, Hongzhi and Nguyen, Quoc Viet Hung},
  journal={ACM Transactions on Intelligent Systems and Technology},
  year={2025}
}

@inproceedings{lee2019wide,
  title={Wide neural networks of any depth evolve as linear models under gradient descent},
  author={Lee, Jaehoon and Xiao, Lechao and Schoenholz, Samuel and Bahri, Yasaman and Novak, Roman and Sohl-Dickstein, Jascha and Pennington, Jeffrey},
  booktitle= {Advances in Neural Information Processing Systems (NeurIPS)},
  year= {2019}
}

@inproceedings{jacot2018neural,
  title={Neural tangent kernel: Convergence and generalization in neural networks},
  author={Jacot, Arthur and Gabriel, Franck and Hongler, Cl{\'e}ment},
  booktitle={Advances in Neural Information Processing Systems (NeurIPS)},
  year={2018}
}

@software{jax2018github,
  author={James Bradbury and Roy Frostig and Peter Hawkins and Matthew James Johnson and Chris Leary and Dougal Maclaurin and George Necula and Adam Paszke and Jake Vander{P}las and Skye Wanderman-{M}ilne and Qiao Zhang},
  title={{JAX}: composable transformations of {P}ython+{N}um{P}y programs},
  url={http://github.com/jax-ml/jax},
  year={2018}
}

@inproceedings{neuraltangents2020,
    title={Neural Tangents: Fast and Easy Infinite Neural Networks in Python},
    author={Roman Novak and Lechao Xiao and Jiri Hron and Jaehoon Lee and Alexander A. Alemi and Jascha Sohl-Dickstein and Samuel S. Schoenholz},
    booktitle={International Conference on Learning Representations (ICLR)},
    year={2020},
    url={https://github.com/google/neural-tangents}
}

@inproceedings{novak2022fast,
    title={Fast Finite Width Neural Tangent Kernel},
    author={Roman Novak and Jascha Sohl-Dickstein and Samuel S. Schoenholz},
    booktitle={International Conference on Machine Learning (ICML)},
    year={2022}
}

@inproceedings{hron2020infinite,
    title={Infinite attention: NNGP and NTK for deep attention networks},
    author={Jiri Hron and Yasaman Bahri and Jascha Sohl-Dickstein and Roman Novak},
    booktitle={International Conference on Machine Learning (ICML)},
    year={2020}
}

@article{sohl2020infinite,
  title={On the infinite width limit of neural networks with a standard parameterization},
  author={Sohl-Dickstein, Jascha and Novak, Roman and Schoenholz, Samuel S and Lee, Jaehoon},
  journal={arXiv preprint},
  year={2020}
}

@inproceedings{han2022fast,
    title={Fast Neural Kernel Embeddings for General Activations},
    author={Insu Han and Amir Zandieh and Jaehoon Lee and Roman Novak and Lechao Xiao and Amin Karbasi},
    booktitle={Advances in Neural Information Processing Systems (NeurIPS)},
    year={2022},
}

@inproceedings{basu2021influence,
  title={Influence Functions in Deep Learning Are Fragile},
  author={Basu, S and Pope, P and Feizi, S},
  booktitle={International Conference on Learning Representations (ICLR)},
  year={2021}
}

@inproceedings{schioppa2022scaling,
  title={Scaling up influence functions},
  author={Schioppa, Andrea and Zablotskaia, Polina and Vilar, David and Sokolov, Artem},
  booktitle={AAAI Conference on Artificial Intelligence},
  year={2022}
}

@inproceedings{guo2020certified,
  title={Certified data removal from machine learning models},
  author={Guo, Chuan and Goldstein, Tom and Hannun, Awni and Van Der Maaten, Laurens},
  booktitle={International Conference on Machine Learning (ICML)},
  year={2020}
}

@article{grosse2023studying,
  title={Studying large language model generalization with influence functions},
  author={Grosse, Roger and Bae, Juhan and Anil, Cem and Elhage, Nelson and Tamkin, Alex and Tajdini, Amirhossein and Steiner, Benoit and Li, Dustin and Durmus, Esin and Perez, Ethan and others},
  journal={arXiv preprint},
  year={2023}
}

@article{lecun2002gradient,
  title={Gradient-based learning applied to document recognition},
  author={LeCun, Yann and Bottou, L{\'e}on and Bengio, Yoshua and Haffner, Patrick},
  journal={Proceedings of the IEEE},
  year={2002}
}

@article{krizhevsky2009learning,
  title={Learning multiple layers of features from tiny images},
  author={Krizhevsky, Alex and Hinton, Geoffrey},
  year={2009}
}
	\bibliographystyle{iclr2027_conference}
	
	\appendix
	
	\newpage
	
	\section{Proofs} \label{A:Proofs}
	
	
	Theorem \ref{thm: dual representation} can be proved by proving the following three propositions. When trained under the risk \( \mathcal{L}_\mathcal{D}(\boldsymbol{\theta})  \), the difference between the parameters obtained by training on any nonempty subset of \( \mathcal{D} \) and those obtained by training on \( \mathcal{D} \) lies in a linear subspace of \( \mathbb{R}^{d_{\boldsymbol{\theta}}} \). This observation is formalized in the following proposition.
	
	\begin{proposition}\label{prop: A.1}
		For any nonempty \( \mathcal{D}_r \subseteq \mathcal{D} \), let \( \hat{\boldsymbol{\theta}}^\star \) and \( \hat{\boldsymbol{\theta}}_r^\star \) denote the model parameters obtained by solving \( \arg \, \min_{\boldsymbol{\theta} \in \mathbb{R}^{d_{\boldsymbol{\theta}}}} \hat{\mathcal{L}}_{\mathcal{D}}(\boldsymbol{\theta}) \) and \( \arg \, \min_{\boldsymbol{\theta} \in \mathbb{R}^{d_{\boldsymbol{\theta}}}} \hat{\mathcal{L}}_{\mathcal{D}_r}(\boldsymbol{\theta}) \), respectively. Then:
		\begin{equation*}
			\hat{\boldsymbol{\theta}}_r^\star - \hat{\boldsymbol{\theta}}^\star =  \nabla_{\boldsymbol{\theta}} f(\mathbf{X}, \boldsymbol{\theta}')^\top \Delta \boldsymbol{\alpha}_r^\star,
		\end{equation*}
		where \( \nabla_{\boldsymbol{\theta}} f(\mathbf{X}, \boldsymbol{\theta}') \in \mathbb{R}^{d_\text{out}|\mathcal{D}| \times d_{\boldsymbol{\theta}}} \) and \( \Delta \boldsymbol{\alpha}_r^\star \in \mathbb{R}^{d_\text{out}|\mathcal{D}|} \) are defined as
		\begin{equation*}
			\nabla_{\boldsymbol{\theta}} f(\mathbf{X}, \boldsymbol{\theta}') \coloneq
			\begin{bmatrix}
				\nabla_{\boldsymbol{\theta}} f(\mathbf{x}^{(1)}, \boldsymbol{\theta}') \\
				\vdots \\
				\nabla_{\boldsymbol{\theta}} f(\mathbf{x}^{(|\mathcal{D}|)}, \boldsymbol{\theta}')
			\end{bmatrix} \quad
			\Delta \boldsymbol{\alpha}_r^\star \coloneq
			\begin{bmatrix}
				\frac{1}{\lambda} \nabla_{f^\text{lin}(\mathbf{X}_f, \hat{\boldsymbol{\theta}}^\star)} \hat{\mathcal{L}}_{\mathcal{D}} \\
				-\frac{1}{\lambda} \left( \nabla_{f^\text{lin}(\mathbf{X}_r, \hat{\boldsymbol{\theta}}_r^\star)} \hat{\mathcal{L}}_{\mathcal{D}_r} - \nabla_{f^\text{lin}(\mathbf{X}_r, \hat{\boldsymbol{\theta}}^\star)} \hat{\mathcal{L}}_{\mathcal{D}} \right) 
			\end{bmatrix}.
		\end{equation*}
	\end{proposition}
	
	\begin{proof}[Proof of Proposition \ref{prop: A.1}]
		Since \( \hat{\boldsymbol{\theta}}^\star \) and \( \hat{\boldsymbol{\theta}}_r^\star \) are the minimizers of \( \arg \, \min_{\boldsymbol{\theta} \in \mathbb{R}^{d_{\boldsymbol{\theta}}}} \hat{\mathcal{L}}_{\mathcal{D}}(\boldsymbol{\theta}) \) and \( \arg \, \min_{\boldsymbol{\theta} \in \mathbb{R}^{d_{\boldsymbol{\theta}}}} \hat{\mathcal{L}}_{\mathcal{D}_r}(\boldsymbol{\theta}) \), respectively, by first-order necessary condition for optimality, they must satisfy
		\begin{equation*}
			\begin{aligned}
				0 &= \nabla_{\boldsymbol{\theta}} \hat{\mathcal{L}}_{\mathcal{D}}(\hat{\boldsymbol{\theta}}^\star) = \nabla_{\boldsymbol{\theta}} f(\mathbf{X}, \boldsymbol{\theta}')^\top \nabla_{f^\text{lin}(\mathbf{X}, \hat{\boldsymbol{\theta}}^\star)} \hat{\mathcal{L}}_{\mathcal{D}} + \lambda (\hat{\boldsymbol{\theta}}^\star - \boldsymbol{\theta}')  \\
				0 &= \nabla_{\boldsymbol{\theta}} \hat{\mathcal{L}}_{\mathcal{D}_r}(\hat{\boldsymbol{\theta}}_r^\star) = \nabla_{\boldsymbol{\theta}} f(\mathbf{X}_r, \boldsymbol{\theta}')^\top \nabla_{f^\text{lin}(\mathbf{X}_r, \hat{\boldsymbol{\theta}}_r^\star)} \hat{\mathcal{L}}_{\mathcal{D}_r} + \lambda (\hat{\boldsymbol{\theta}}_r^\star - \boldsymbol{\theta}'),
			\end{aligned}
		\end{equation*}
		where \( \nabla_{\boldsymbol{\theta}} f(\mathbf{X}, \boldsymbol{\theta}') \in \mathbb{R}^{d_\text{out}|\mathcal{D}| \times d_{\boldsymbol{\theta}}} \) and \( \nabla_{f^\text{lin}(\mathbf{X}, \hat{\boldsymbol{\theta}}^\star)} \hat{\mathcal{L}}_{\mathcal{D}} \in \mathbb{R}^{d_\text{out}|\mathcal{D}|} \) are defined as
		\begin{equation*}
			\nabla_{\boldsymbol{\theta}} f(\mathbf{X}, \boldsymbol{\theta}') \coloneq
			\begin{bmatrix}
				\nabla_{\boldsymbol{\theta}} f(\mathbf{x}^{(1)}, \boldsymbol{\theta}') \\
				\vdots \\
				\nabla_{\boldsymbol{\theta}} f(\mathbf{x}^{(|\mathcal{D}|)}, \boldsymbol{\theta}')
			\end{bmatrix} \quad \nabla_{f^\text{lin}(\mathbf{X}, \hat{\boldsymbol{\theta}}^\star)} \hat{\mathcal{L}}_{\mathcal{D}} \coloneq \frac{1}{|\mathcal{D}|}
			\begin{bmatrix}
				\nabla_f\ell(f^\text{lin}(\mathbf{x}^{(1)}, \hat{\boldsymbol{\theta}}^\star), \mathbf{y}^{(1)}) \\
				\vdots \\
				\nabla_f\ell(f^\text{lin}(\mathbf{x}^{(|\mathcal{D}|)}, \hat{\boldsymbol{\theta}}^\star), \mathbf{y}^{(|\mathcal{D}|)})
			\end{bmatrix},
		\end{equation*}
		and \( \nabla_{\boldsymbol{\theta}} f(\mathbf{X}_r, \boldsymbol{\theta}') \) and \( \nabla_{f^\text{lin}(\mathbf{X}_r, \hat{\boldsymbol{\theta}}_r^\star)} \hat{\mathcal{L}}_{\mathcal{D}_r} \) are defined in the same way. Rearranging terms yields
		\begin{equation*}
			\begin{aligned}
				\hat{\boldsymbol{\theta}}^\star &= \nabla_{\boldsymbol{\theta}} f(\mathbf{X}, \boldsymbol{\theta}')^\top \boldsymbol{\alpha}^\star + \boldsymbol{\theta}' \\
				\hat{\boldsymbol{\theta}}_r^\star &= \nabla_{\boldsymbol{\theta}} f(\mathbf{X}_r, \boldsymbol{\theta}')^\top \boldsymbol{\alpha}_r^\star + \boldsymbol{\theta}',
			\end{aligned}
		\end{equation*}
		where \( \boldsymbol{\alpha}^\star \coloneq -\frac{1}{\lambda} \nabla_{f^\text{lin}(\mathbf{X}, \hat{\boldsymbol{\theta}}^\star)} \hat{\mathcal{L}}_{\mathcal{D}} \in \mathbb{R}^{d_\text{out}|\mathcal{D}|} \) and \( \boldsymbol{\alpha}_r^\star \coloneq  -\frac{1}{\lambda} \nabla_{f^\text{lin}(\mathbf{X}_r, \hat{\boldsymbol{\theta}}_r^\star)} \hat{\mathcal{L}}_{\mathcal{D}_r} \in \mathbb{R}^{d_\text{out}|\mathcal{D}_r|} \). Assume, without loss of generality, that data points in \( \mathcal{D}_f \) are the first \( |\mathcal{D}_f| \) data points in \( \mathcal{D} \), then by defining \( \tilde{\boldsymbol{\alpha}}_r^\star \coloneq \begin{bmatrix}
			\mathbf{0}^\top, \boldsymbol{\alpha}_r^{\star \top}
		\end{bmatrix}^\top \in \mathbb{R}^{d_\text{out}|\mathcal{D}|} \), \( \hat{\boldsymbol{\theta}}_r^\star \) can be rewritten as \( \hat{\boldsymbol{\theta}}_r^\star = \nabla_{\boldsymbol{\theta}} f(\mathbf{X}, \boldsymbol{\theta}')^\top \tilde{\boldsymbol{\alpha}}_r^\star + \boldsymbol{\theta}' \). Subtracting \( \hat{\boldsymbol{\theta}}^\star \) from \( \hat{\boldsymbol{\theta}}_r^\star \) yields
		\begin{equation*}
			\hat{\boldsymbol{\theta}}_r^\star - \hat{\boldsymbol{\theta}}^\star = \nabla_{\boldsymbol{\theta}} f(\mathbf{X}, \boldsymbol{\theta}')^\top \Delta \boldsymbol{\alpha}_r^\star,
		\end{equation*}
		where
		\begin{equation*}
			\Delta \boldsymbol{\alpha}_r^\star = \tilde{\boldsymbol{\alpha}}_r^\star - \boldsymbol{\alpha}^\star =
			\begin{bmatrix}
				\frac{1}{\lambda} \nabla_{f^\text{lin}(\mathbf{X}_f, \hat{\boldsymbol{\theta}}^\star)} \hat{\mathcal{L}}_{\mathcal{D}}\\
				-\frac{1}{\lambda} \left( \nabla_{f^\text{lin}(\mathbf{X}_r, \hat{\boldsymbol{\theta}}_r^\star)} \hat{\mathcal{L}}_{\mathcal{D}_r} - \nabla_{f^\text{lin}(\mathbf{X}_r, \hat{\boldsymbol{\theta}}^\star)} \hat{\mathcal{L}}_{\mathcal{D}} \right)
			\end{bmatrix}.
		\end{equation*}
	\end{proof}
	
	Proposition \ref{prop: A.1} implies that, for every nonempty \( \mathcal{D}_r \subseteq \mathcal{D} \), the solution to \( \arg \, \min_{\boldsymbol{\theta} \in \mathbb{R}^{d_{\boldsymbol{\theta}}}} \hat{\mathcal{L}}_{\mathcal{D}_r}(\boldsymbol{\theta}) \) always belongs to the convex subspace \( \mathcal{C} \coloneq \hat{\boldsymbol{\theta}}^\star + \operatorname{col} \left( \nabla_{\boldsymbol{\theta}} f(\mathbf{X}, \boldsymbol{\theta}')^\top \right) \) of \( \mathbb{R}^{d_{\boldsymbol{\theta}}} \), where \( \operatorname{col}(A) \) denotes the column space of \( A \). Therefore, the original unconstrained problem can be reformulated as a constrained one, and the following proposition shows that the two problems are equivalent in the sense that they yield the same unique minimizer.
	
	\begin{proposition}
		\label{prop: A.2}
		Assume \( \ell \) is convex with respect to model output, then for any nonempty \( \mathcal{D}_r \subseteq \mathcal{D} \), \( \arg \, \min_{\boldsymbol{\theta} \in \mathcal{C}} \hat{\mathcal{L}}_{\mathcal{D}_r}(\boldsymbol{\theta}) \) and \( \arg \, \min_{\boldsymbol{\theta} \in \mathbb{R}^{d_{\boldsymbol{\theta}}}} \hat{\mathcal{L}}_{\mathcal{D}_r}(\boldsymbol{\theta}) \) have the same unique minimizer.
	\end{proposition}
	
	\begin{proof}[Proof of Proposition \ref{prop: A.2}]
		By definition of minimizer and the fact that \( \mathcal{C} \subset \mathbb{R}^{d_{\boldsymbol{\theta}}} \), we know \( \hat{\boldsymbol{\theta}}_r^\star \coloneq \arg \, \min_{\boldsymbol{\theta} \in \mathbb{R}^{d_{\boldsymbol{\theta}}}} \hat{\mathcal{L}}_{\mathcal{D}_r}(\boldsymbol{\theta}) \) and \( \hat{\boldsymbol{\theta}}_{r,c}^\star \coloneq \arg \, \min_{\boldsymbol{\theta} \in \mathcal{C}} \hat{\mathcal{L}}_{\mathcal{D}_r}(\boldsymbol{\theta}) \) satisfy
		\begin{equation*}
			\begin{aligned}
				&\hat{\mathcal{L}}_{\mathcal{D}_r}(\hat{\boldsymbol{\theta}}_r^\star) \leq \hat{\mathcal{L}}_{\mathcal{D}_r}(\boldsymbol{\theta}), \ \forall \boldsymbol{\theta} \in \mathcal{C} \\
				&\hat{\mathcal{L}}_{\mathcal{D}_r}(\hat{\boldsymbol{\theta}}_{r,c}^\star) \leq \hat{\mathcal{L}}_{\mathcal{D}_r}(\boldsymbol{\theta}), \ \forall \boldsymbol{\theta} \in \mathcal{C}.
			\end{aligned}
		\end{equation*}
		Thus, \( \hat{\boldsymbol{\theta}}_r^\star \) and \( \hat{\boldsymbol{\theta}}_{r,c}^\star \) are both minimizers of \( \arg \, \min_{\boldsymbol{\theta} \in \mathcal{C}} \hat{\mathcal{L}}_{\mathcal{D}_r}(\boldsymbol{\theta}) \). When \( \ell \) is convex with respect to the model output, \( \arg \, \min_{\boldsymbol{\theta} \in \mathcal{C}} \hat{\mathcal{L}}_{\mathcal{D}_r}(\boldsymbol{\theta}) \) is a convex optimization problem with a strictly convex objective function. As a result, the global minimizer is unique, and hence \( \hat{\boldsymbol{\theta}}_r^\star = \hat{\boldsymbol{\theta}}_{r,c}^\star \).
	\end{proof}
	
	To solve the constrained problem, one can introduce a reparameterization that enforces the constraint, and reformulates the problem as an unconstrained one. We enforce the constraint \( \boldsymbol{\theta} \in \mathcal{C} \) via the reparameterization 
	\begin{equation*}
			\boldsymbol{\theta} \coloneq \phi(\Delta \boldsymbol{\alpha}) \coloneq \hat{\boldsymbol{\theta}}^\star + \nabla_{\boldsymbol{\theta}} f\left(\mathbf{X}, \boldsymbol{\theta}'\right)^\top \Delta \boldsymbol{\alpha},
		\end{equation*}
	where \( \Delta \boldsymbol{\alpha} \in \mathbb{R}^{d_\text{out}|\mathcal{D}|} \). Under the reparameterization, \( f^\text{lin}(\cdot, \boldsymbol{\theta}) \) becomes \( g^\text{lin}(\cdot, \Delta \boldsymbol{\alpha}) \coloneq f^\text{lin}(\cdot, \hat{\boldsymbol{\theta}}^\star) + \mathbf{K}(\cdot, \mathbf{X} )\Delta \boldsymbol{\alpha} \) and \( \hat{\mathcal{L}}_{\mathcal{D}}(\boldsymbol{\theta}) \) becomes
	\begin{equation*}
			\tilde{\mathcal{L}}_{\mathcal{D}}(\Delta\boldsymbol{\alpha}) \coloneq \frac{1}{|\mathcal{D}|} \sum_{(\mathbf{x}, \mathbf{y}) \in \mathcal{D}} \Big ( \ell(g^\text{lin}(\mathbf{x}, \Delta \boldsymbol{\alpha}), \mathbf{y}) + \frac{\lambda}{2} \| \phi(\Delta \boldsymbol{\alpha}) - \boldsymbol{\theta}' \|_2^2 \Big ),
		\end{equation*}
	where \( \mathbf{K}(\cdot, \cdot) \coloneq \nabla_{\boldsymbol{\theta}} f(\cdot, \boldsymbol{\theta}') \nabla_{\boldsymbol{\theta}} f(\cdot, \boldsymbol{\theta}')^\top \) denotes the empirical Neural Tangent Kernel (NTK) \citep{jacot2018neural}, and \( \mathbf{K}(\mathbf{X}_1, \mathbf{X}_2) \in \mathbb{R}^{d_\text{out}N_1 \times d_\text{out}N_2} \) for feature matrices \( \mathbf{X}_1 \in \mathbb{R}^{N_1 \times d_\text{in}} \) and \( \mathbf{X}_2 \in \mathbb{R}^{N_2 \times d_\text{in}} \). \( \tilde{\mathcal{L}}_{\mathcal{D}}(\Delta\boldsymbol{\alpha}) \) is a dual representation of \( \hat{\mathcal{L}}_{\mathcal{D}}(\boldsymbol{\theta}) \) and an unconstrained formulation of the constrained problem. The relationship between the solutions to the unconstrained problem in \( \Delta \boldsymbol{\alpha} \) and the constrained problem in \( \boldsymbol{\theta} \) is stated in the following proposition.
	
	\begin{proposition}
		\label{prop: A.3}
		If \( \ell \) is convex with respect to model output, \( \lambda > 0 \), and \( \nabla_{\boldsymbol{\theta}}f\left(\mathbf{X}, \boldsymbol{\theta}'\right) \) has full row rank, then for any nonempty \( \mathcal{D}_r \subseteq \mathcal{D} \), \( \arg \min_{\boldsymbol{\Delta \alpha} \in \mathbb{R}^{d_\text{out} |\mathcal{D}|}} \tilde{\mathcal{L}}_{\mathcal{D}_r}(\Delta \boldsymbol{\alpha}) \) and \( \arg \, \min_{\boldsymbol{\theta} \in \mathcal{C}} \hat{\mathcal{L}}_{\mathcal{D}_r}(\boldsymbol{\theta}) \) have unique minimizers \( \boldsymbol{\Delta \alpha}_r^\star \) and \( \hat{\boldsymbol{\theta}}_{r,c}^\star \) that are related by
		\begin{equation*}
			\boldsymbol{\Delta \alpha}_r^\star \xrightleftharpoons[\phi^{-1}(\cdot)]{\phi(\cdot)} \hat{\boldsymbol{\theta}}_{r,c}^\star.
		\end{equation*}
	\end{proposition}
	
	\begin{proof}[Proof of Proposition \ref{prop: A.3}]
		When \( \ell \) is convex with respect to the model output and \( \lambda > 0 \), both \( \arg \min_{\boldsymbol{\Delta \alpha} \in \mathbb{R}^{d_\text{out} |\mathcal{D}|}} \tilde{\mathcal{L}}_{\mathcal{D}_r}(\Delta \boldsymbol{\alpha}) \) and \( \arg \, \min_{\boldsymbol{\theta} \in \mathcal{C}} \hat{\mathcal{L}}_{\mathcal{D}_r}(\boldsymbol{\theta}) \) are convex optimization problems with strictly convex objective functions. As a result, their minimizers \( \boldsymbol{\Delta \alpha}_r^\star \) and \( \hat{\boldsymbol{\theta}}_{r,c}^\star \) are unique. If \( \nabla_{\boldsymbol{\theta}}f\left(\mathbf{X}, \boldsymbol{\theta}'\right) \) has full row rank, then \( \phi(\cdot) \) is injective and for every \( \boldsymbol{\theta} \in \operatorname{Im}(\phi) \) there exists exactly one \( \Delta \boldsymbol{\alpha} \) such that \( \phi(\Delta \boldsymbol{\alpha}) = \boldsymbol{\theta} \). Since, by Propositions \ref{prop: A.1} and \ref{prop: A.2}, we know \( \hat{\boldsymbol{\theta}}_{r,c}^\star = \phi(\Delta \boldsymbol{\alpha}_r^\star) \) and \( \hat{\boldsymbol{\theta}}_{r,c}^\star \in \operatorname{Im}(\phi) \), we can conclude that \( \boldsymbol{\Delta \alpha}_r^\star \xrightleftharpoons[\phi^{-1}(\cdot)]{\phi(\cdot)} \hat{\boldsymbol{\theta}}_{r,c}^\star \).
	\end{proof}
	
	Proposition \ref{prop: A.2} and \ref{prop: A.3} together imply that, for linearized models and any nonempty \( \mathcal{D}_r \subseteq \mathcal{D} \), solving \( \arg \ \min_{\boldsymbol{\Delta \alpha} \in \mathbb{R}^{d_\text{out} |\mathcal{D}|}} \tilde{\mathcal{L}}_{\mathcal{D}_r}(\Delta \boldsymbol{\alpha}) \) is equivalent to solving \( \arg \ \min_{\boldsymbol{\theta} \in \mathbb{R}^{d_{\boldsymbol{\theta}}}} \hat{\mathcal{L}}_{\mathcal{D}_r}(\boldsymbol{\theta}) \), in the sense that both lead to the same parameter \( \hat{\boldsymbol{\theta}}_{r,c}^\star \coloneq \arg \, \min_{\boldsymbol{\theta} \in \mathcal{C}} \hat{\mathcal{L}}_{\mathcal{D}_r}(\boldsymbol{\theta}) \) and the same empirical risk \( \hat{\mathcal{L}}_{\mathcal{D}_r}(\hat{\boldsymbol{\theta}}_{r,c}^\star) \).
	
	\newpage
	
	\section{Derivations} \label{A:Derivations}
	
	
	We present detailed derivations for the equations shown in Section \ref{subsec:Influence Functions} and \ref{subsec:Influence on Delta alpha}.
	
	\subsection{Derivation for (\ref{eq:influence on theta})}
	\label{subsec:Derivations for influence on theta}
	The following derivation, adapted from \citep{koh2017understanding}, serves to introduce our notation and to generalize their single data point removal setting to a multiple data points removal setting.
	
	In the multiple data points removal setting, we define the upweighted empirical risk as \( \mathcal{L}_{\mathcal{D}}(\boldsymbol{\theta}, \epsilon) \coloneqq \mathcal{L}_{\mathcal{D}}(\boldsymbol{\theta}) + |\mathcal{D}_f|\mathcal{L}_{\mathcal{D}_f}(\boldsymbol{\theta}) \epsilon \). Substituting \( \mathcal{L}_{\mathcal{D}}(\boldsymbol{\theta}) \) and \( \mathcal{L}_{\mathcal{D}_f}(\boldsymbol{\theta}) \) with their respective definitions and setting \( \epsilon = -\frac{1}{|\mathcal{D}|} \) yields \( \mathcal{L}_{\mathcal{D}}(\boldsymbol{\theta}, -\frac{1}{|\mathcal{D}|}) = \frac{|\mathcal{D}_r|}{|\mathcal{D}|} \mathcal{L}_{\mathcal{D}_r}(\boldsymbol{\theta}) \). Since \( \mathcal{L}_{\mathcal{D}_r}(\boldsymbol{\theta}) \) is strongly convex and \(  \mathcal{L}_{\mathcal{D}}(\boldsymbol{\theta}, -\frac{1}{|\mathcal{D}|}) \) is a scalar multiple of \( \mathcal{L}_{\mathcal{D}_r}(\boldsymbol{\theta}) \), the two optimization problems \( \arg \, \min_{\boldsymbol{\theta} \in \mathbb{R}^{d_{\boldsymbol{\theta}}}} \mathcal{L}_{\mathcal{D}_r}(\boldsymbol{\theta}) \) and \( \arg \, \min_{\boldsymbol{\theta} \in \mathbb{R}^{d_{\boldsymbol{\theta}}}} \mathcal{L}_{\mathcal{D}}(\boldsymbol{\theta}, -\frac{1}{|\mathcal{D}|}) \) share the same unique minimizer, i.e., \( \boldsymbol{\theta}_r^\star = \boldsymbol{\theta}^\star(-\frac{1}{|\mathcal{D}|}) \). 
	
	By first order necessary condition for optimality, \( \boldsymbol{\theta}^\star \) and \( \boldsymbol{\theta}^\star(\epsilon) \) satisfy \( \nabla_{\boldsymbol{\theta}} \mathcal{L}_{\mathcal{D}}(\boldsymbol{\theta}^\star) = \mathbf{0} \) and \( \nabla_{\boldsymbol{\theta}} \mathcal{L}_{\mathcal{D}}(\boldsymbol{\theta}^\star(\epsilon), \epsilon) = \mathbf{0} \). When \( \epsilon \) is small, \( \boldsymbol{\theta}^\star(\epsilon) \) is close to \(  \boldsymbol{\theta}^\star \) and \( \nabla_{\boldsymbol{\theta}} \mathcal{L}_{\mathcal{D}}(\boldsymbol{\theta}^\star(\epsilon), \epsilon) \) can be accurately approximated by its first-order Taylor expansion around \( \boldsymbol{\theta}^\star \), so the optimality condition for \( \boldsymbol{\theta}^\star(\epsilon) \) becomes \( \nabla_{\boldsymbol{\theta}} \mathcal{L}_\mathcal{D}(\boldsymbol{\theta}^\star, \epsilon) + \nabla_{\boldsymbol{\theta}}^2\mathcal{L}_\mathcal{D}(\boldsymbol{\theta}^\star, \epsilon) \left( \boldsymbol{\theta}^\star(\epsilon) - \boldsymbol{\theta}^\star \right) \approx \mathbf{0} \). Rearranging terms and taking first- and second-order derivatives of \( \mathcal{L}_{\mathcal{D}}(\boldsymbol{\theta}, \epsilon) \) with respect to \( \boldsymbol{\theta} \) yields
	\begin{equation*}
		\begin{aligned}
			\boldsymbol{\theta}^\star(\epsilon) &\approx \boldsymbol{\theta}^\star - \nabla_{\boldsymbol{\theta}}^2\mathcal{L}_\mathcal{D}(\boldsymbol{\theta}^\star, \epsilon)^{-1} \nabla_{\boldsymbol{\theta}} \mathcal{L}_\mathcal{D}(\boldsymbol{\theta}^\star, \epsilon) \\
			&= \boldsymbol{\theta}^\star - \left( \nabla_{\boldsymbol{\theta}}^2 \mathcal{L}_{\mathcal{D}}(\boldsymbol{\theta}^\star) + |\mathcal{D}_f| \nabla_{\boldsymbol{\theta}}^2 \mathcal{L}_{\mathcal{D}_f}(\boldsymbol{\theta}^\star) \epsilon \right)^{-1} |\mathcal{D}_f| \nabla_{\boldsymbol{\theta}} \mathcal{L}_{\mathcal{D}_f}(\boldsymbol{\theta}^\star) \epsilon.
		\end{aligned}
	\end{equation*}
	When \( \epsilon = -\frac{1}{|\mathcal{D}|} \), using the relation  \( \boldsymbol{\theta}_r^\star = \boldsymbol{\theta}^\star(-\frac{1}{|\mathcal{D}|}) \), the above equation becomes \( \boldsymbol{\theta}_r^\star \approx \boldsymbol{\theta}^\star + \frac{|\mathcal{D}_f|}{|\mathcal{D}|} \mathbf{H}_{\boldsymbol{\theta}^\star}^{-1} \nabla_{\boldsymbol{\theta}}\mathcal{L}_{\mathcal{D}_f}(\boldsymbol{\theta}^\star) \), where \( \mathbf{H}_{\boldsymbol{\theta}^\star}^{-1} \coloneqq ( \nabla_{\boldsymbol{\theta}}^2 \mathcal{L}_{\mathcal{D}}(\boldsymbol{\theta}^\star) - \frac{|\mathcal{D}_f|}{|\mathcal{D}|} \nabla_{\boldsymbol{\theta}}^2 \mathcal{L}_{\mathcal{D}_f}(\boldsymbol{\theta}^\star) )^{-1} \) can be further approximated by \( \nabla_{\boldsymbol{\theta}}^2\mathcal{L}_\mathcal{D}(\boldsymbol{\theta}^\star)^{-1} \) when \( \frac{|\mathcal{D}_f|}{|\mathcal{D}|} \) is small, due to the identity \( (\mathbf{A} + \epsilon \mathbf{X})^{-1} = \mathbf{A}^{-1} - \epsilon \mathbf{A}^{-1} \mathbf{X} \mathbf{A}^{-1} + \mathcal{O}(\epsilon^2) \).
	
	\subsection{Derivation for (\ref{eq:influence on delta alpha})}
	
	The derivation of \eqref{eq:influence on delta alpha} is identical to that of \eqref{eq:influence on theta}, with only a few minor tweaks. Specifically, by replacing \( \mathcal{L}_{\mathcal{D}}(\boldsymbol{\theta}, \epsilon) \) with \( \tilde{\mathcal{L}}_{\mathcal{D}}(\Delta\boldsymbol{\alpha}, \epsilon) \coloneq \tilde{\mathcal{L}}_{\mathcal{D}}(\Delta \boldsymbol{\alpha}) + |\mathcal{D}_f| \tilde{\mathcal{L}}_{\mathcal{D}_f}(\Delta \boldsymbol{\alpha}) \epsilon\), but otherwise following the same derivation as in Section \ref{subsec:Derivations for influence on theta}, one can show that
	\begin{equation*}
		\Delta \boldsymbol{\alpha}_r^\star \approx \Delta \boldsymbol{\alpha}^\star + \frac{|\mathcal{D}_f|}{|\mathcal{D}|} \mathbf{H}_{\Delta \boldsymbol{\alpha}^\star}^{-1} \nabla_{\Delta \boldsymbol{\alpha}} \tilde{\mathcal{L}}_{\mathcal{D}_f}(\Delta \boldsymbol{\alpha}^\star),
	\end{equation*}
	where \( \mathbf{H}_{\Delta \boldsymbol{\alpha}^\star}^{-1} \coloneq \left( \nabla_{\boldsymbol{\Delta \alpha}}^2 \tilde{\mathcal{L}}_{\mathcal{D}}(\Delta \boldsymbol{\alpha}^\star) - \frac{|\mathcal{D}_f|}{|\mathcal{D}|} \nabla_{\Delta \boldsymbol{\alpha}}^2 \tilde{\mathcal{L}}_{\mathcal{D}_f}(\Delta \boldsymbol{\alpha}^\star) \right)^{-1} \) can be further approximated by \( \nabla_{\boldsymbol{\Delta \alpha}}^2 \tilde{\mathcal{L}}_{\mathcal{D}}(\Delta \boldsymbol{\alpha}^\star)^{-1} \) when \( \frac{|\mathcal{D}_f|}{|\mathcal{D}|} \) is small. To obtain \eqref{eq:influence on delta alpha}, it remains to show that \( \Delta \boldsymbol{\alpha}^\star = \mathbf{0} \), which follows from Propositions \ref{prop: A.2} and \ref{prop: A.3}. Specifically, these two propositions imply that, \( \Delta \boldsymbol{\alpha}^\star \), \( \hat{\boldsymbol{\theta}}^\star \) and \( \hat{\boldsymbol{\theta}}_c^\star \coloneq \arg \, \min_{\boldsymbol{\theta} \in \mathcal{C}} \hat{\mathcal{L}}_{\mathcal{D}}(\boldsymbol{\theta}) \) are related by
	\begin{equation*}
		\boldsymbol{\Delta \alpha}^\star \xrightleftharpoons[\phi^{-1}(\cdot)]{\phi(\cdot)} \hat{\boldsymbol{\theta}}_c^\star = \hat{\boldsymbol{\theta}}^\star.
	\end{equation*}
	Therefore, \( \hat{\boldsymbol{\theta}}^\star = \hat{\boldsymbol{\theta}}^\star + \nabla_{\boldsymbol{\theta}} f\left(\mathbf{X}, \boldsymbol{\theta}'\right)^\top \Delta \boldsymbol{\alpha}^\star \) and rearranging terms yields \( \nabla_{\boldsymbol{\theta}} f\left(\mathbf{X}, \boldsymbol{\theta}'\right)^\top \Delta \boldsymbol{\alpha}^\star = \mathbf{0} \). Since \( \nabla_{\boldsymbol{\theta}} f\left(\mathbf{X}, \boldsymbol{\theta}'\right) \) is assumed to have full row rank, \( \nabla_{\boldsymbol{\theta}} f\left(\mathbf{X}, \boldsymbol{\theta}'\right)^\top \) has full column rank and a trivial null space, i.e., \( \operatorname{ker}(\nabla_{\boldsymbol{\theta}} f\left(\mathbf{X}, \boldsymbol{\theta}'\right)^\top) = \{ \mathbf{0} \} \), which implies that \( \Delta \boldsymbol{\alpha}^\star = \mathbf{0} \).
	
	\subsection{Derivation for Equations (\ref{eq:Hessian with repsect to delta alpha star}) and (\ref{eq:grad of reparameterized loss at 0})}
	
	For any dataset \( \mathcal{D}_i \), by definition of \( \tilde{\mathcal{L}}_{\mathcal{D}_i}(\Delta\boldsymbol{\alpha}) \), \( \nabla_{\Delta \boldsymbol{\alpha}} \tilde{\mathcal{L}}_{\mathcal{D}_i}(\Delta \boldsymbol{\alpha}) \) and \(  \nabla_{\Delta \boldsymbol{\alpha}}^2 \tilde{\mathcal{L}}_{\mathcal{D}_i}(\Delta \boldsymbol{\alpha}) \) are given by
	\begin{align}
		&\nabla_{\Delta \boldsymbol{\alpha}} \tilde{\mathcal{L}}_{\mathcal{D}_i}(\Delta \boldsymbol{\alpha}) = \mathbf{K}(\mathbf{X}, \mathbf{X}_i) \nabla_{g^\text{lin}(\mathbf{X}_i, \Delta \boldsymbol{\alpha})} \tilde{\mathcal{L}}_{\mathcal{D}_i} + \lambda \nabla_{\boldsymbol{\theta}} f(\mathbf{X}, \boldsymbol{\theta}') (\hat{\boldsymbol{\theta}}^\star - \boldsymbol{\theta}')  + \lambda \mathbf{K} \Delta \boldsymbol{\alpha} \label{eq:grad of reparameterized loss} \\
		&\nabla_{\Delta \boldsymbol{\alpha}}^2 \tilde{\mathcal{L}}_{\mathcal{D}_i}(\Delta \boldsymbol{\alpha}) = \mathbf{K}(\mathbf{X}, \mathbf{X}_i) \nabla_{g^\text{lin}(\mathbf{X}_i, \Delta \boldsymbol{\alpha})}^2 \tilde{\mathcal{L}}_{\mathcal{D}_i} \mathbf{K}(\mathbf{X}_i, \mathbf{X}) + \lambda \mathbf{K}, \label{eq:Hessian of reparameterized loss}
	\end{align}
	where \( \mathbf{K}(\mathbf{X}, \mathbf{X}_i) \coloneq \begin{bmatrix}\mathbf{K}(\mathbf{X}, \mathbf{x}^{(1)}), \cdots, \mathbf{K}(\mathbf{X}, \mathbf{x}^{(|\mathcal{D}_i|)})\end{bmatrix} \in \mathbb{R}^{d_\text{out}|\mathcal{D}| \times d_\text{out}|\mathcal{D}_i|} \),
	\begin{equation*}
		\nabla_{g^\text{lin}(\mathbf{X}_i, \Delta \boldsymbol{\alpha})} \tilde{\mathcal{L}}_{\mathcal{D}_i} \coloneq \frac{1}{|\mathcal{D}_i|}
		\begin{bmatrix}
			\nabla_f \ell(g^\text{lin}(\mathbf{x}^{(1)}, \Delta \boldsymbol{\alpha}), \mathbf{y}^{(1)}) \\
			\vdots \\
			\nabla_f \ell(g^\text{lin}(\mathbf{x}^{(|\mathcal{D}_i|)}, \Delta \boldsymbol{\alpha}), \mathbf{y}^{(|\mathcal{D}_i|)})
		\end{bmatrix} \in \mathbb{R}^{d_\text{out}|\mathcal{D}_i|}.
	\end{equation*}
	and \( \nabla_{g^\text{lin}(\mathbf{X}_i, \Delta \boldsymbol{\alpha})}^2 \tilde{\mathcal{L}}_{\mathcal{D}_i} \in \mathbb{R}^{d_\text{out}|\mathcal{D}_i| \times d_\text{out}|\mathcal{D}_i|} \) is a block diagonal matrix whose blocks are
	\begin{equation*}
		\frac{1}{|\mathcal{D}_i|} \nabla_f^2 \ell(g^\text{lin}(\mathbf{x}, \Delta \boldsymbol{\alpha}), \mathbf{y}) \in \mathbb{R}^{d_\text{out} \times d_\text{out}}, \ \forall (\mathbf{x}, \mathbf{y}) \in \mathcal{D}_i.
	\end{equation*}
	\( \nabla_{\Delta \boldsymbol{\alpha}} \tilde{\mathcal{L}}_{\mathcal{D}_f}(\Delta \boldsymbol{\alpha}) \), \( \nabla_{\Delta \boldsymbol{\alpha}}^2 \tilde{\mathcal{L}}_{\mathcal{D}_f}(\Delta \boldsymbol{\alpha}) \) and \( \nabla_{\Delta \boldsymbol{\alpha}}^2 \tilde{\mathcal{L}}_{\mathcal{D}}(\Delta \boldsymbol{\alpha}) \) are obtained by replacing \( \mathcal{D}_i \) in \eqref{eq:grad of reparameterized loss} and \eqref{eq:Hessian of reparameterized loss} with \( \mathcal{D}_f \) and \( \mathcal{D} \).
	
	Equation \eqref{eq:grad of reparameterized loss at 0} is obtained by evaluating \( \nabla_{\Delta \boldsymbol{\alpha}} \tilde{\mathcal{L}}_{\mathcal{D}_f}(\Delta \boldsymbol{\alpha}) \) at \( \Delta \boldsymbol{\alpha}^\star = \mathbf{0} \),
	\begin{equation*}
		\begin{aligned}
			\nabla_{\Delta \boldsymbol{\alpha}} \tilde{\mathcal{L}}_{\mathcal{D}_f}(\mathbf{0}) &= \mathbf{K}(\mathbf{X}, \mathbf{X}_f) \nabla_{g^\text{lin}(\mathbf{X}_f, \mathbf{0})} \tilde{\mathcal{L}}_{\mathcal{D}_f} + \lambda \nabla_{\boldsymbol{\theta}} f(\mathbf{X}, \boldsymbol{\theta}') (\hat{\boldsymbol{\theta}}^\star - \boldsymbol{\theta}')  \\
			&= \mathbf{K}(\mathbf{X}_f, \mathbf{X})^\top \nabla_{f^\text{lin}(\mathbf{X}_f, \hat{\boldsymbol{\theta}}^\star)} \hat{\mathcal{L}}_{\mathcal{D}_f} + \lambda \mathbf{K} \boldsymbol{\alpha}^\star,
		\end{aligned}
	\end{equation*}
	where we have used the fact that \( \hat{\boldsymbol{\theta}}^\star - \boldsymbol{\theta}'  = \nabla_{\boldsymbol{\theta}} f(\mathbf{X}, \boldsymbol{\theta}')^\top \boldsymbol{\alpha}^\star \) and \( g^\text{lin}(\cdot, \mathbf{0}) = f^\text{lin}(\cdot, \hat{\boldsymbol{\theta}}^\star) \). By definition of \( \mathbf{H}_{\Delta \boldsymbol{\alpha}} \), we have
	\begin{equation*}
		\begin{aligned}
			\mathbf{H}_{\Delta \boldsymbol{\alpha}} &= \nabla_{\boldsymbol{\Delta \alpha}}^2 \tilde{\mathcal{L}}_{\mathcal{D}}(\Delta \boldsymbol{\alpha}) - \frac{|\mathcal{D}_f|}{|\mathcal{D}|} \nabla_{\Delta \boldsymbol{\alpha}}^2 \tilde{\mathcal{L}}_{\mathcal{D}_f}(\Delta \boldsymbol{\alpha})\\
			&= \frac{|\mathcal{D}_r|}{|\mathcal{D}|} \left( \mathbf{K}(\mathbf{X}_r, \mathbf{X})^\top \nabla_{g^\text{lin}(\mathbf{X}_r, \Delta \boldsymbol{\alpha})}^2 \tilde{\mathcal{L}}_{\mathcal{D}_r} \mathbf{K}(\mathbf{X}_r, \mathbf{X}) + \lambda \mathbf{K} \right),
		\end{aligned}
	\end{equation*}
	evaluating \( \mathbf{H}_{\Delta \boldsymbol{\alpha}} \) at \( \Delta \boldsymbol{\alpha}^\star = \mathbf{0} \) and using the equality \( g^\text{lin}(\cdot, \mathbf{0}) = f^\text{lin}(\cdot, \hat{\boldsymbol{\theta}}^\star) \) gives \eqref{eq:Hessian with repsect to delta alpha star}.
	
	\subsection{Derivation for (\ref{eq:change in output delta alpha}) and (\ref{eq:change in loss delta alpha})}
	
	The first-order Taylor expansion of \( g^\text{lin}(\mathbf{x}_t, \Delta \boldsymbol{\alpha}_r^\star) \) and \( \tilde{\mathcal{L}}_{\{ \mathbf{z}_t \}}(\Delta \boldsymbol{\alpha}_r^\star) \) around \( \Delta \boldsymbol{\alpha}^\star=0 \) gives
	\begin{equation*}
		\begin{aligned}
			&g^\text{lin}(\mathbf{x}_t, \Delta \boldsymbol{\alpha}_r^\star) - g^\text{lin}(\mathbf{x}_t, \mathbf{0}) = \mathbf{K}(\mathbf{x}_t, \mathbf{X}) \Delta \boldsymbol{\alpha}_r^\star \\
			&\tilde{\mathcal{L}}_{\{ \mathbf{z}_t \}}(\Delta \boldsymbol{\alpha}_r^\star) - \tilde{\mathcal{L}}_{\{ \mathbf{z}_t \}}(\mathbf{0}) \approx \nabla_{\Delta \boldsymbol{\alpha}} \tilde{\mathcal{L}}_{\{ \mathbf{z}_t \}}(\mathbf{0})^\top \Delta \boldsymbol{\alpha}_r^\star.
		\end{aligned}
	\end{equation*}
	When \( g^\text{lin}(\mathbf{x}_t, \Delta \boldsymbol{\alpha}) \) and \( \tilde{\mathcal{L}}_{\{ \mathbf{z}_t \}}(\Delta \boldsymbol{\alpha}) \) are evaluated at \( \mathbf{0} \), by their definitions, we have \( g^\text{lin}(\mathbf{x}_t, \mathbf{0}) =  f^\text{lin}(\mathbf{x}_t, \hat{\boldsymbol{\theta}}^\star) \) and
	\begin{equation*}
		\tilde{\mathcal{L}}_{\{ \mathbf{z}_t \}}(\mathbf{0}) = \ell(f^\text{lin}(\mathbf{x}_t, \hat{\boldsymbol{\theta}}^\star), \mathbf{y}_t) + \frac{\lambda}{2} \| \hat{\boldsymbol{\theta}}^\star - \boldsymbol{\theta}' \|_2^2 = \hat{\mathcal{L}}_{\{ \mathbf{z}_t \}}(\hat{\boldsymbol{\theta}}^\star).
	\end{equation*}
	When they are evaluated at \( \Delta \boldsymbol{\alpha}_r^\star \), by Proposition \ref{prop: A.1} and definition of \( f^\text{lin}(\cdot, \boldsymbol{\theta}) \), we have
	\begin{equation*}
		\begin{aligned}
			g^\text{lin}(\mathbf{x}_t, \Delta \boldsymbol{\alpha}_r^\star) &= f^\text{lin}(\mathbf{x}_t, \hat{\boldsymbol{\theta}}^\star) + \mathbf{K}(\mathbf{x}_t, \mathbf{X}) (\tilde{\boldsymbol{\alpha}}_r^\star - \boldsymbol{\alpha}^\star) \\
			&= f^\text{lin}(\mathbf{x}_t, \hat{\boldsymbol{\theta}}^\star) + \nabla_{\boldsymbol{\theta}}f(\mathbf{x}_t, \boldsymbol{\theta}') (\hat{\boldsymbol{\theta}}_r^\star - \boldsymbol{\theta}')  - \nabla_{\boldsymbol{\theta}}f(\mathbf{x}_t, \boldsymbol{\theta}') (\hat{\boldsymbol{\theta}}^\star - \boldsymbol{\theta}')  \\
			&= f(\mathbf{x}_t, \boldsymbol{\theta}') + \nabla_{\boldsymbol{\theta}}f(\mathbf{x}_t, \boldsymbol{\theta}') (\hat{\boldsymbol{\theta}}_r^\star -  \boldsymbol{\theta}') \\
			&= f^\text{lin}(\mathbf{x}_t, \hat{\boldsymbol{\theta}}_r^\star),
		\end{aligned}
	\end{equation*}
	and
	\begin{equation*}
		\begin{aligned}
			\tilde{\mathcal{L}}_{\{ \mathbf{z}_t \}}(\Delta \boldsymbol{\alpha}_r^\star) &= \ell(f^\text{lin}(\mathbf{x}_t, \hat{\boldsymbol{\theta}}_r^\star), \mathbf{y}_t) + \frac{\lambda}{2} \| \hat{\boldsymbol{\theta}}^\star + \nabla_{\boldsymbol{\theta}}f(\mathbf{X}, \boldsymbol{\theta}')^\top \Delta \boldsymbol{\alpha}_r^\star - \boldsymbol{\theta}' \|_2^2 \\
			&= \ell(f^\text{lin}(\mathbf{x}_t, \hat{\boldsymbol{\theta}}_r^\star), \mathbf{y}_t) + \frac{\lambda}{2} \| \hat{\boldsymbol{\theta}}_r^\star - \boldsymbol{\theta}' \|_2^2 \\
			& = \hat{\mathcal{L}}_{\{ \mathbf{z}_t \}}(\hat{\boldsymbol{\theta}}_r^\star).
		\end{aligned}
	\end{equation*}
	Combining these results gives \eqref{eq:change in output delta alpha} and \eqref{eq:change in loss delta alpha}.
	
	\subsection{Derivation for (\ref{eq:change in loss vectorization})}
	
	Evaluating \eqref{eq:change in loss delta alpha} at all data points in \( \mathcal{D}_t \) is equivalent to performing the matrix vector product:
	\begin{equation*}
		\begin{bmatrix}
			\left( \mathbf{K}(\mathbf{x}_t^{(1)}, \mathbf{X})^\top \nabla_{f^\text{lin}(\mathbf{x}_t^{(1)}, \hat{\boldsymbol{\theta}}^\star)} \hat{\mathcal{L}}_{\{ \mathbf{z}_t^{(1)} \}} + \lambda \mathbf{K} \boldsymbol{\alpha}^\star \right)^\top \\
			\vdots \\
			\left( \mathbf{K}(\mathbf{x}_t^{(|\mathcal{D}_t|)}, \mathbf{X})^\top \nabla_{f^\text{lin}(\mathbf{x}_t^{(|\mathcal{D}_t|)}, \hat{\boldsymbol{\theta}}^\star)} \hat{\mathcal{L}}_{\{ \mathbf{z}_t^{(|\mathcal{D}_t|)} \}} + \lambda \mathbf{K} \boldsymbol{\alpha}^\star \right)^\top
		\end{bmatrix}
		\Delta \boldsymbol{\alpha}_r^\star.
	\end{equation*}
	Assume \( f^\text{lin} \) has scalar-valued output, the above product can be rewritten as
	\begin{equation*}
		\begin{bmatrix}
			\nabla_{f^\text{lin}(\mathbf{x}_t^{(1)}, \hat{\boldsymbol{\theta}}^\star)} \hat{\mathcal{L}}_{\{ \mathbf{z}_t^{(1)} \}} \\
			\vdots \\
			\nabla_{f^\text{lin}(\mathbf{x}_t^{(|\mathcal{D}_t|)}, \hat{\boldsymbol{\theta}}^\star)} \hat{\mathcal{L}}_{\{ \mathbf{z}_t^{(|\mathcal{D}_t|)} \}}
		\end{bmatrix}
		\odot
		\left(
		\begin{bmatrix}
			\mathbf{K}(\mathbf{x}_t^{(1)}, \mathbf{X}) \\
			\vdots \\
			\mathbf{K}(\mathbf{x}_t^{(|\mathcal{D}_t|)}, \mathbf{X})
		\end{bmatrix}
		\Delta \boldsymbol{\alpha}_r^\star \right) + \lambda 
		\begin{bmatrix}
			\boldsymbol{\alpha}^{\star \top} \mathbf{K} \Delta \boldsymbol{\alpha}_r^\star \\
			\vdots \\
			\boldsymbol{\alpha}^{\star \top} \mathbf{K} \Delta \boldsymbol{\alpha}_r^\star
		\end{bmatrix},
	\end{equation*}
	which gives \eqref{eq:change in loss vectorization}. The generalization to vector-valued \( f^\text{lin} \) follows by noticing
	\begin{equation*}
		\mathbf{1}^\top \left( \nabla_{f^\text{lin}(\mathbf{x}_t, \hat{\boldsymbol{\theta}}^\star)} \hat{\mathcal{L}}_{\{ \mathbf{z}_t \}} \odot \left(  \mathbf{K}(\mathbf{x}_t, \mathbf{X})\Delta \boldsymbol{\alpha}_r^\star \right) \right) = \nabla_{f^\text{lin}(\mathbf{x}_t, \hat{\boldsymbol{\theta}}^\star)} \hat{\mathcal{L}}_{\{ \mathbf{z}_t \}}^\top \mathbf{K}(\mathbf{x}_t, \mathbf{X})\Delta \boldsymbol{\alpha}_r^\star,
	\end{equation*}
	where \( \mathbf{1} \in \mathbb{R}^{d_\text{out}} \) denotes the all-ones vector.

\end{document}